\pdfoutput=1
\documentclass[hidelinks,onefignum,onetabnum]{siamart251216}

\usepackage[T1]{fontenc}
\usepackage{lmodern}
\usepackage{amssymb}
\usepackage{mathtools}
\usepackage{enumitem}
\usepackage{graphicx}
\usepackage{float}
\usepackage{algpseudocode}
\usepackage{microtype}
\usepackage[title]{appendix}
\usepackage{needspace}
\headers{A Wasserstein Trust-Region Method}{Y. Wan, T. Gao, J. Duan, and J. Sun}
\title{A Wasserstein Trust-Region Method for Smooth Potential-Interaction Energies}
\author{You Wan\thanks{School of Mathematics, Nanjing University, Nanjing 210093, China; and School of Sciences, Guangdong Provincial Key Laboratory of Mathematical and Neural Dynamical Systems, Great Bay University, Dongguan 523000, China (\email{you.wan@smail.nju.edu.cn}).}
\and Ting Gao\thanks{School of Mathematics and Statistics, Center for Mathematical Sciences, and Steklov-Wuhan Institute for Mathematical Exploration, Huazhong University of Science and Technology, Wuhan 430074, China (corresponding author, \email{gao0716@hust.edu.cn}).}
\and Jinqiao Duan\thanks{School of Sciences, Guangdong Provincial Key Laboratory of Mathematical and Neural Dynamical Systems, Great Bay University, Dongguan 523000, China (\email{duan@gbu.edu.cn}).}
\and Jie Sun\thanks{Leibniz Theory Lab, 2012 Labs, Huawei Technologies Co. Ltd., Hong Kong SAR, China (\email{riosun@gmail.com}).}}

\newcommand{\R}{\mathbb R}
\newcommand{\Pcal}{\mathcal P_2}
\newcommand{\Fcal}{\mathcal F}
\newcommand{\Vcal}{\mathcal V}
\newcommand{\Wcal}{\mathcal W}
\newcommand{\gradW}{\nabla_W}

\newcommand{\Bop}{\mathsf B}
\DeclareMathOperator{\id}{Id}
\newsiamremark{remark}{Remark}
\newsiamthm{assumption}{Assumption}
\newcommand{\normrho}[2]{\|#1\|_{#2}}
\newcommand{\iprho}[3]{\langle #1,#2\rangle_{#3}}

\begin{document}
\maketitle
\ifpdf
\hypersetup{
  pdftitle={A Wasserstein Trust-Region Method for Smooth Potential-Interaction Energies},
  pdfauthor={You Wan, Ting Gao, Jinqiao Duan, Jie Sun},
  pdfsubject={Wasserstein trust-region methods},
  pdfkeywords={Wasserstein space, trust-region methods, optimization over measures, interaction energies, particle methods, global convergence, iteration complexity}
}
\fi

\begin{abstract}
  Finding low-energy particle configurations and approximating probability distributions lead to minimizing potential-interaction energies in Wasserstein space. Second-order information can speed up minimizing these possibly nonconvex energies, but a quadratic model is accurate only near the current measure. We introduce a Wasserstein trust-region (WTR) method for smooth energies of this type. Each iteration transports the current measure $\rho$ along a displacement field $s$, producing $(\id+s)_\#\rho$. The field approximately minimizes a quadratic model built from the second variation of the energy, subject to $\|s\|_{L^2(\rho)}\le\Delta$, using a truncated conjugate-gradient method. The ratio of actual to predicted decrease determines step acceptance and updates the trust-region radius. If the energy is bounded below and the potential and the interaction kernel have globally bounded Hessians, we prove that the energy is nonincreasing, the Wasserstein-gradient norms converge to zero, and an $\varepsilon$-stationary iterate is reached within $O(\varepsilon^{-2})$ iterations, including rejected ones. If the potential is quadratically coercive, every weak accumulation point is stationary. For empirical measures, the iteration is carried out on the particle positions, and the complexity bound is independent of the number of particles when $\Fcal(\rho_0)-\inf\Fcal$ is bounded uniformly in that number. For $2048$ and $4096$ particles with Gaussian pair repulsion in a two-dimensional harmonic trap, WTR reaches the gradient tolerance $2\times10^{-4}$ faster than limited-memory BFGS (L-BFGS), with a median speedup of $1.40$. For equal-weight integration rules on two non-Gaussian targets, WTR reaches an integration error of $10^{-4}$ about $1.8$ times faster than L-BFGS and $1.3$ times faster than a Krylov trust-region solver.
\end{abstract}



\section{Introduction}
\label{sec:intro}

Optimization over probability measures provides a common
framework for interacting-particle models and distributional
approximation. We consider smooth potential--interaction energies
of the form
\[
  \Fcal(\rho)
  =
  \int_{\mathbb R^d} V(x)\,\rho(dx)
  +
  \frac12
  \int_{\mathbb R^d\times\mathbb R^d}
  W(x-y)\,\rho(dx)\rho(dy),
  \qquad
  \rho\in\mathcal P_2(\mathbb R^d)
\]
where $\mathcal P_2(\mathbb R^d)$ denotes the probability measures
with finite second moment, $V$ is an external potential, and $W$
is an even interaction kernel. These energies need not be
displacement convex. Second-order information can accelerate
their minimization, but a quadratic model of the energy is
accurate only near the current measure. This paper develops and
analyzes a Wasserstein trust-region (WTR) method, which restricts
each step to a region where this model can be trusted.

Molecular and soft-matter models provide one motivation.
Coarse graining replaces macromolecules by soft particles
with effective pair interactions. The Gaussian core model is a
standard coarse-grained model with a bounded, penetrable Gaussian
repulsion \cite{LouisBolhuisHansen2000,LouisBolhuisHansenMeijer2000}.
For flexible fourth-generation dendrimers, Gaussian effective
interactions have been derived from monomer-density and
small-angle neutron-scattering data \cite{Likos2002}.
Near its equilibrium position, an optical trap acts as a linear
spring, so its potential is locally harmonic \cite{NeumanBlock2004};
harmonic trapping has also been combined with ultrasoft Gaussian
interactions in a one-dimensional colloidal model
\cite{GernertKlapp2015}.
These sources motivate the model used in
Section~\ref{sec:numerics}.

A second motivation is the approximation of a target probability
measure by minimizing a kernel discrepancy. For a smooth
translation-invariant positive-definite kernel, half of the squared
maximum mean discrepancy (MMD) between a candidate measure and
a fixed target differs from a potential--interaction energy only
by a target-dependent constant \cite{Arbel2019}. The target
determines the external potential, while the kernel determines
the interaction term. The resulting particle measure also defines an
equal-weight integration rule. Stationary MMD points have recently been
studied for numerical integration \cite{Chen2026}.

Minimizers of potential--interaction energies also arise in
nonlocal models of swarming and aggregation, in which $W$ combines
attraction and repulsion, with or without an external potential
\cite{TopazBertozzi2004,BernoffTopaz2011}. In these models, stationary
states and local minimizers can concentrate on rings, on other
lower-dimensional sets, or on finitely many points, depending on the
behavior of $W$ near the origin \cite{Kolokolnikov2011,Balague2013},
and conditions on $W$ are known under which compactly supported global
minimizers exist \cite{Canizo2015}. Particle approximations of such
states are commonly computed by integrating the gradient flow of the
particle energy to steady state. The method developed here
computes them by minimizing the particle energy with second-order
information, and the
experiments of Section~\ref{sec:numerics} include states concentrated
on rings and on finitely many points.

There have been several approaches exploiting the geometry of Wasserstein space
for optimization. The Jordan--Kinderlehrer--Otto (JKO) scheme
\cite{JKO1998} gives a proximal formulation of gradient flows,
while Wasserstein proximal-gradient algorithms combine explicit
and implicit schemes \cite{SalimKorbaLuise2020,DiMarino2025}.
Bonet et al.~\cite{Bonet2024} develop mirror and preconditioned
gradient methods and derive second-variation formulas along
pushforward curves for potential and interaction energies.
Second-order approaches include Stein variational Newton and
information Newton methods, which use damping or regularization
\cite{Detommaso2018,WangLi2020}.
Yamamoto et al.\ perturb a Wasserstein gradient flow along
directions of negative Hessian curvature to escape saddle points
\cite{Yamamoto2025}, and Lascu and Suzuki propose a regularized
saddle-free Newton method with a local linear rate
\cite{LascuSuzuki2026}.
Gaussian variational methods exploit the Bures--Wasserstein
geometry of covariance matrices \cite{Lambert2022,Diao2023}.
Optimal-transport discrepancies have also been used as
policy-update constraints in reinforcement learning \cite{Terpin2022}.

Trust-region methods are the classical way to use a local
quadratic model only where it is accurate, in Euclidean and Hilbert
spaces \cite{Toint1988,NocedalWright2006,ConnGouldToint2000},
with extensions to Riemannian manifolds \cite{Absil2007}
and to metric spaces through model oracles
\cite{Manns2026}.
In the present setting, a trial displacement $s$ generates
the candidate measure $(\id+s)_\#\rho$, and the trust-region
constraint bounds $\|s\|_{L^2(\rho)}$.
This norm bounds the Wasserstein distance between $\rho$ and
the candidate measure.
A ratio of actual to predicted decrease determines acceptance
and updates the radius.

The main analytical difficulty is that the underlying space
changes with the iterate.
The quadratic model at $\rho$ is defined on $L^2(\rho;\mathbb R^d)$, so
the models and gradients at successive iterates belong to different
Hilbert spaces, and every constant must be independent of the current
measure.
Moreover, the analysis must accommodate atomic measures if it
is to apply directly to particle implementations.
We bound the model operator
and the model error by constants that are the same for every
measure with finite second moment, compare gradient norms at
successive iterates, and bound the total number of iterations,
accepted and rejected.
Because these estimates hold for every measure with finite second
moment, they apply unchanged to empirical measures.

Throughout the analysis, $V$ and $W$ are twice continuously
differentiable and bounded below, $W$ is even, and their Hessians
are globally bounded. Under these hypotheses, the main contributions are
as follows.

\begin{itemize}[leftmargin=1.6em]
  \item  \textbf{A trust-region framework in Wasserstein space.}
  The method minimizes smooth, possibly nonconvex potential--interaction energies. Each step moves the measure by a displacement field that approximately minimizes a quadratic model of the energy, subject to a bound on its $L^2(\rho)$ norm. The method applies to every measure with finite second moment, including empirical measures.

  \item  \textbf{Convergence and complexity guarantees with constants independent of the measure.}
  The bounds on the model operator and on the model error are the same
  for every measure with finite second moment. With them we prove that
  the energy is nonincreasing, that an $\varepsilon$-stationary iterate
  is reached within $O(\varepsilon^{-2})$ iterations, accepted and
  rejected, and that the Wasserstein-gradient norms converge to zero,
  where gradients at successive iterates lie in different Hilbert
  spaces. If $V$ is quadratically coercive, every weak accumulation
  point is stationary. For empirical measures, the constants in the
  complexity bound do not depend on the number of particles.

  \item  \textbf{Numerical performance.}
  \begin{itemize}[leftmargin=1.4em]
    \item  \textbf{Trapped Gaussian core model.}
    For the two-dimensional Gaussian core model in a harmonic trap with
    $2048$ and $4096$ particles, WTR reaches the gradient tolerance $2\times10^{-4}$ faster than L-BFGS, with a median speedup of $1.40$.
    At the largest trap stiffness, $97\%$ of the particles lie on one ring
    and the rest at the center; the radius and the mass fraction agree
    with the minimizer of the mean-field energy among such measures.

    \item  \textbf{Equal-weight integration rules.}
    For the two-moons and ring targets, each represented by a large
    sample, WTR constructs equal-weight integration rules that reach a
    root-mean-square error of $10^{-4}$ about $1.8$ times faster than L-BFGS and $1.3$
    times faster than a Krylov trust-region solver.

    \item  \textbf{Dependence on the number of particles.}
    Initialized near a stationary measure supported on three points, WTR
    reaches the gradient tolerance $10^{-6}$ in two or three iterations
    for every particle number from $10^2$ to $10^6$, whereas L-BFGS needs
    twelve.
  \end{itemize}
\end{itemize}

The remainder of the paper is organized as follows.
Section~\ref{sec:method} defines the pushforward update and the
quadratic model and states the algorithm.
Section~\ref{sec:global} proves that the energy is nonincreasing,
bounds the number of iterations needed to bring the gradient norm
below a given tolerance $\varepsilon$, shows that the gradient norm
tends to zero and that weak accumulation points are stationary, and
shows that a variant restricted to the tangent space retains these
guarantees.
Section~\ref{sec:numerics} applies the method to particle systems
and reports experiments.
Section~\ref{sec:discussion} concludes.
Appendix~\ref{app:convolution} proves that the interaction field
$\nabla W*\mu$ depends continuously on $\mu$ under weak convergence,
and Appendix~\ref{app:steihaug} states the truncated-CG subsolver and
proves that its step is feasible and achieves the Cauchy decrease.

\section{The Wasserstein trust-region method}
\label{sec:method}

\subsection{Problem}

We study the problem $\min_{\rho\in\Pcal(\R^d)}\Fcal(\rho)$ for the potential--interaction energy $\Fcal(\rho)=\int V\,d\rho+\frac12\iint W(x-y)\,d\rho(x)\,d\rho(y)$ of Section~\ref{sec:intro}, where $\Pcal(\R^d)$ denotes the probability measures with finite second moment. The potential $V$ assigns an energy to each particle position, and the kernel $W$ assigns an interaction energy to each pair of positions. For smooth $V$ and $W$, the Wasserstein gradient field is $\nabla V+\nabla W*\rho$. This field generates a first-order Wasserstein gradient flow and provides a natural stationarity criterion \cite{AGS2008,Lanzetti2024}. The Wasserstein second variation supplies curvature information \cite{Otto2001,Gigli2012}.

\subsection{Preliminaries}

Let $\Pcal(\R^d)$ denote the probability measures with finite second moment, equipped with the quadratic Wasserstein distance $W_2$. The Otto tangent space at $\rho$ is the $L^2(\rho)$-closure of gradient fields,
\[
  T_\rho\Pcal=\overline{\{\nabla\phi:\phi\in C_c^\infty(\R^d)\}}^{\,L^2(\rho;\R^d)},\qquad \iprho{u}{v}{\rho}=\int u\cdot v\,d\rho,\quad \normrho{s}{\rho}=\iprho{s}{s}{\rho}^{1/2}
\]
The \textbf{Wasserstein gradient} $\gradW\Fcal(\rho)=\nabla\frac{\delta\Fcal}{\delta\rho}\in T_\rho\Pcal$ is the Riesz representative of the first variation in this inner product; see \cite{Lanzetti2024} for a modern variational framework. The Otto (Wasserstein) Hessian is the second derivative of $\Fcal$ along Wasserstein geodesics, viewed as a quadratic form on $T_\rho\Pcal$ \cite{Otto2001,Gigli2012}; we use it only through \eqref{eq:ambient-second-variation} below.
We write $\rho_j\rightharpoonup\rho$ for weak convergence of probability measures, $\|B\|$ for the operator norm on $L^2(\rho)$, and $K*\rho(x)=\int K(x-y)\,d\rho(y)$. We call $\rho$ \emph{stationary} when $\|\gradW\Fcal(\rho)\|_\rho=0$.

For each displacement field $s\in L^2(\rho;\R^d)$, we define the pushforward update
\[
  R_\rho(s):=(\id+s)_\#\rho
\]
We use three elementary properties of this update.
\begin{itemize}[leftmargin=1.4em]
  \item \textbf{Displacement control.} The measure $(\id,\id+s)_\#\rho$ is a coupling of $\rho$ and $R_\rho(s)$, hence $W_2(\rho,R_\rho(s))\le\normrho{s}{\rho}$.
  \item \textbf{Preservation of finite second moments.} $\int|x+s(x)|^2\,d\rho\le 2\int|x|^2\,d\rho+2\normrho{s}{\rho}^2<\infty$, so $R_\rho(s)\in\Pcal$.
  \item \textbf{First-variation identity.} $R_\rho(0)=\rho$ and, under Assumption~\ref{ass:smooth-class}, differentiation under the integral gives $\frac{d}{dt}\Fcal(R_\rho(ts))\big|_0=\iprho{\gradW\Fcal(\rho)}{s}{\rho}$, so a quadratic model with linear term $\iprho{\gradW\Fcal(\rho)}{s}{\rho}$ agrees with $t\mapsto\Fcal(R_\rho(ts))$ to first order.
\end{itemize}

\label{sec:operator}

For the potential--interaction class in Assumption~\ref{ass:smooth-class}, define $\Bop_\rho$ on $L^2(\rho;\mathbb R^d)$ by
\begin{equation}
  \label{eq:ambient-operator}
  (\Bop_\rho s)(x)
  :=\nabla^2V(x)s(x)
  +\int\nabla^2W(x-y)\bigl(s(x)-s(y)\bigr)\,d\rho(y)
\end{equation}
Because $W$ is even, we have the symmetric bilinear identity
\begin{align}
  \iprho{u}{\Bop_\rho v}{\rho}
  ={} & \int u^\top\nabla^2Vv\,d\rho\notag                    \\
      & +\frac12\iint\bigl(u(x)-u(y)\bigr)^\top\nabla^2W(x-y)
  \bigl(v(x)-v(y)\bigr)\,d\rho(x)d\rho(y)
  \label{eq:ambient-bilinear}
\end{align}
Thus $\Bop_\rho$ is self-adjoint. The corresponding potential and interaction second-variation formulas also appear in \cite[Appendix~C.2]{Bonet2024}. Direct differentiation gives
\begin{equation}
  \label{eq:ambient-second-variation}
  \iprho{s}{\Bop_\rho s}{\rho}
  =\left.\frac{d^2}{dt^2}\Fcal\bigl((\id+ts)_\#\rho\bigr)\right|_{t=0}
\end{equation}
If $s=\nabla\phi$ is a smooth gradient field, the curve $t\mapsto(\id+t\nabla\phi)_\#\rho$ is a Wasserstein geodesic for small $t$, and \eqref{eq:ambient-second-variation} is then the Otto Hessian of $\Fcal$ evaluated in the direction $s$. Thus $\Bop_\rho$ extends the Otto Hessian from tangent directions to all of $L^2(\rho;\R^d)$; projecting it back onto $T_\rho\Pcal$ gives the tangent-space operator used in Section~\ref{sec:tangent-variant}.
\begin{remark}[Search directions outside the tangent space]
  \label{rem:ambient}
  The trust-region subproblem in Algorithm~\ref{alg:wtr} below is posed on all of $L^2(\rho;\R^d)$, not only on the tangent space $T_\rho\Pcal$, because conjugate-gradient directions can leave the tangent space after one step. The first direction $p_0=-g$, with $g=\gradW\Fcal(\rho)$, is a gradient field, but the next residual has the form $g+\alpha\Bop_\rho g$ with $\alpha>0$ and contains the term $\alpha\nabla^2V\,g$, and a field of the form $\nabla^2V\,\nabla\phi$ is in general not a gradient. The pushforward update and the second variation are defined for every $s\in L^2(\rho;\R^d)$, so the method can use such directions directly. Section~\ref{sec:tangent-variant} gives a variant that instead projects each product $\Bop_\rho p$ back onto $T_\rho\Pcal$.
\end{remark}

\subsection{Wasserstein trust-region algorithm}

Given the current iterate $\rho_k$, set $g_k=\gradW\Fcal(\rho_k)$ and $B_k=\Bop_{\rho_k}$. Form the quadratic model
\[
  m_k(s)=\Fcal(\rho_k)+\iprho{g_k}{s}{\rho_k}+\tfrac12\iprho{s}{B_ks}{\rho_k}
\]
Fix a CG-iteration limit $J\in\{1,2,\ldots\}$ and a relative residual tolerance $\omega\in[0,1)$. At each iteration, use the Steihaug--Toint subsolver in Appendix~\ref{app:steihaug} to approximately solve the subproblem $\min_{\normrho{s}{\rho_k}\le\Delta_k}m_k(s)$. The returned step has at least the \textbf{Cauchy decrease}, namely the reduction in $m_k$ obtained by minimizing it over the feasible part of the negative-gradient ray. In particular,
\[
  m_k(0)-m_k(s_k)\ \ge\ \tfrac12\normrho{g_k}{\rho_k}\,\min\!\Big(\tfrac{\normrho{g_k}{\rho_k}}{\|B_k\|},\,\Delta_k\Big)
\]
with the convention $a/0=+\infty$ for $a>0$. Appendix~\ref{app:steihaug} gives the update and truncation rules and proves feasibility and the Cauchy bound.
We compute the ratio $\varrho_k=\dfrac{\Fcal(\rho_k)-\Fcal(R_{\rho_k}(s_k))}{m_k(0)-m_k(s_k)}$ of actual to predicted decrease and accept the step when $\varrho_k\ge\eta_1$. We update the radius with parameters $0<\gamma_1<1<\gamma_2$, $\theta\in(0,1)$, and $\widehat\Delta\in(0,\infty]$: shrink the radius after a rejection, enlarge it after a very successful step that nearly reaches the boundary, and otherwise keep it,
\[
  \Delta_{k+1}=\begin{cases}\gamma_1\Delta_k,&\varrho_k<\eta_1\\ \min(\gamma_2\Delta_k,\widehat\Delta),&\varrho_k\ge\eta_2\text{ and }\normrho{s_k}{\rho_k}\ge\theta\Delta_k\\ \Delta_k,&\text{otherwise}\end{cases}
\]
Algorithm~\ref{alg:wtr} summarizes one iteration. Each iteration is either accepted or rejected; the truncated-CG steps inside it are called \emph{CG iterations}. We assume $0<\Delta_0\le\widehat\Delta$ and stop if $g_k=0$. If the method terminates, extend the sequence by setting $\rho_j=\rho_k$ for $j\ge k$.

\begin{algorithm}[ht]
  \caption{One iteration of the Wasserstein trust-region method}
  \label{alg:wtr}
  \begin{algorithmic}[1]
    \Require $\rho_k$, $\Delta_k$; $0<\eta_1<\eta_2<1$, $0<\gamma_1<1<\gamma_2$, $\theta\in(0,1)$, $\widehat\Delta$; $J\in\{1,2,\ldots\}$, $\omega\in[0,1)$
    \State $g_k \gets \gradW\Fcal(\rho_k)$; \textbf{if} $g_k=0$ \textbf{then stop}; $B_k \gets \Bop_{\rho_k}$
    \State $s_k \gets$ Steihaug--Toint truncated-CG step for $m_k$ on $\|s\|_{\rho_k}\le\Delta_k$ with $J,\omega$ \Comment{Appendix~\ref{app:steihaug}}
    \State $\rho_k^+ \gets R_{\rho_k}(s_k) = (\id+s_k)_\#\rho_k$ \Comment{pushforward update}
    \State $\varrho_k \gets \dfrac{\Fcal(\rho_k)-\Fcal(\rho_k^+)}{m_k(0)-m_k(s_k)}$
    \If{$\varrho_k \ge \eta_1$}
    \State $\rho_{k+1} \gets \rho_k^+$ \Comment{accept}
    \Else
    \State $\rho_{k+1} \gets \rho_k$ \Comment{reject}
    \EndIf
    \State update $\Delta_{k+1}$ from $\Delta_k$, $\varrho_k$, $\|s_k\|_{\rho_k}$ by the rule above
  \end{algorithmic}
\end{algorithm}

\section{Global convergence}
\label{sec:global}

This section establishes first-order convergence guarantees for the Wasserstein trust-region method on the smooth potential--interaction class specified below. Uniform Hessian bounds yield estimates for the model operator and model error that hold for every probability measure with finite second moment. Together with the Cauchy decrease and the radius-update rule, these estimates give an $O(\varepsilon^{-2})$ bound on the number of iterations required to reach an $\varepsilon$-stationary iterate. We then compare gradients across successive pushforward updates to prove that the Wasserstein-gradient norms converge to zero. Under quadratic coercivity of the potential, every weak accumulation point is stationary. Finally, we show that projection of the model operator onto the Wasserstein tangent space preserves these guarantees with the same constants.

\subsection{Assumptions and preparatory estimates}
\label{sec:functional-class}

\begin{assumption}[Smooth potential--interaction class]
  \label{ass:smooth-class}
  Let
  \[
    \Fcal(\rho)=\underbrace{\int V\,d\rho}_{\Vcal}+\underbrace{\tfrac12\iint W(x-y)\,d\rho(x)\,d\rho(y)}_{\Wcal}
  \]
  where $V,W\in C^2(\R^d)$, $W$ is even,
  \[
    \|\nabla^2V\|_\infty\le C_V,\qquad \|\nabla^2W\|_\infty\le C_W
  \]
  and $V\ge\underline V$, $W\ge\underline W>-\infty$. The Hessian bounds imply linear growth of $\nabla V$ and $\nabla W$; hence $\gradW\Fcal(\rho)=\nabla V+\nabla W*\rho\in L^2(\rho;\R^d)$ for every $\rho\in\Pcal$.
\end{assumption}

\begin{proposition}[Uniform bound on second derivatives along pushforward curves]
  \label{prop:bddhess}
  Under Assumption~\ref{ass:smooth-class}, for every $\rho\in\Pcal$ and $s\in L^2(\rho;\R^d)$, the map $h:\R\to\R$, $h(t):=\Fcal(R_\rho(ts))$, is $C^2$ and satisfies
  \[
    |h''(t)|\le\delta_2\,\normrho{s}{\rho}^2,\qquad \delta_2:=C_V+C_W
  \]
\end{proposition}

\begin{proof}
  For the potential part $\Vcal(R_\rho(ts))=\int V(x+ts(x))\,d\rho$, the uniform bound on $\nabla^2V$ and dominated convergence justify differentiating under the integral, giving
  \[
    \tfrac{d^2}{dt^2}\Vcal(R_\rho(ts))=\int s^\top\nabla^2V(x+t s(x))\,s\,d\rho,\qquad \Big|\tfrac{d^2}{dt^2}\Vcal\Big|\le C_V\normrho{s}{\rho}^2
  \]
  For the interaction part $\Wcal(R_\rho(ts))=\tfrac12\iint W(x-y+t(s(x)-s(y)))\,d\rho\,d\rho$,
  \[
    \tfrac{d^2}{dt^2}\Wcal=\tfrac12\iint(s(x)-s(y))^\top\nabla^2W(x-y+t(s(x)-s(y)))(s(x)-s(y))\,d\rho\,d\rho
  \]
  and the variance identity $\iint|s(x)-s(y)|^2\,d\rho\,d\rho=2\normrho{s}{\rho}^2-2\big|\!\int s\,d\rho\big|^2\le 2\normrho{s}{\rho}^2$ bounds this by $C_W\normrho{s}{\rho}^2$. Summing yields $|h''(t)|\le(C_V+C_W)\normrho{s}{\rho}^2$.
\end{proof}

Write $m_\rho$ for the quadratic model at $\rho$. The proposition has two consequences, valid on all of $L^2(\rho;\R^d)$ and with constants independent of $\rho$:

\begin{enumerate}[leftmargin=1.6em,label=(\arabic*)]
  \item \textbf{Bounded model operator.} The quadratic-form bound and self-adjointness imply $\|\Bop_\rho\|\le\delta_2$, uniformly in $\rho$.
  \item \textbf{Quadratic model error.} $\big|\Fcal(R_\rho(s))-m_\rho(s)\big|=\big|\int_0^1(1-t)(h''(t)-h''(0))\,dt\big|\le\delta_2\normrho{s}{\rho}^2$.
\end{enumerate}
By definition, a measure is stationary precisely when $\nabla V+\nabla W*\rho=0$ $\rho$-almost everywhere. If $\delta_2=0$, then $V$ and $W$ are affine and bounded below, hence constant, so every measure is stationary and the algorithm stops at once. We assume $\delta_2>0$ from now on.

\subsection{Iteration complexity}

\begin{theorem}[Iteration complexity]
  \label{thm:global}
  Under Assumption~\ref{ass:smooth-class}, run Algorithm~\ref{alg:wtr} with $B_k=\Bop_{\rho_k}$ from any $\rho_0\in\Pcal$. Then $\Fcal(\rho_k)$ is nonincreasing and
  \[
    \liminf_{k\to\infty}\, \normrho{\gradW\Fcal(\rho_k)}{\rho_k}=0
  \]
  For the first iteration index $K_\varepsilon$ with $\normrho{\gradW\Fcal(\rho_{K_\varepsilon})}{\rho_{K_\varepsilon}}\le\varepsilon$, one has $K_\varepsilon=O(\varepsilon^{-2})$.
\end{theorem}

\begin{proof}
  The proof has three steps. If the radius is small enough, the iteration is accepted, so the radius is bounded below by a positive constant. Each accepted iteration then decreases the energy by at least an amount depending only on $\varepsilon$. Finally, each rejection shrinks the radius by the factor $\gamma_1$ and each acceptance enlarges it by at most $\gamma_2$; since the radius is bounded below, the number of rejected iterations is bounded in terms of the number of accepted ones. The uniform model estimates make every constant independent of $\rho_k$. Let $g_k=\gradW\Fcal(\rho_k)$, $G_k=\|g_k\|_{\rho_k}$, and $p_k:=m_k(0)-m_k(s_k)$ for the predicted decrease. If $G_k=0$, the algorithm stops. Otherwise Proposition~\ref{prop:steihaug-cauchy} gives the Cauchy decrease
  \begin{equation}
    \label{eq:cauchy-detailed}
    p_k\ge\frac12G_k\min\!\left(\frac{G_k}{\delta_2},\Delta_k\right)>0
  \end{equation}
  Since $m_k(0)=\Fcal(\rho_k)$, the model-error estimate gives the explicit ratio bound
  \begin{equation}
    \label{eq:ratio-detailed}
    1-\varrho_k
    =\frac{\Fcal(R_{\rho_k}(s_k))-m_k(s_k)}{p_k}
    \le\frac{|\Fcal(R_{\rho_k}(s_k))-m_k(s_k)|}{p_k}
    \le\frac{\delta_2\Delta_k^2}{p_k}
  \end{equation}
  If $\Delta_k\le G_k/\delta_2$, combining the two estimates yields
  \begin{equation}
    \label{eq:small-radius-accepted}
    1-\varrho_k\le\frac{2\delta_2\Delta_k}{G_k}
  \end{equation}
  Let $\alpha=(1-\eta_1)/2$. Thus $\Delta_k\le\alpha G_k/\delta_2$ implies $\varrho_k\ge\eta_1$: every iteration with such a radius is accepted.

  Fix $\varepsilon>0$ and consider an initial segment of iterations $0,1,\ldots,\bar k$ on which $G_k>\varepsilon$. Let $K_\varepsilon$ denote the first index with $G_{K_\varepsilon}\le\varepsilon$. Define
  \begin{equation}
    \label{eq:radius-floor}
    \overline\Delta_\varepsilon:=\alpha\frac{\varepsilon}{\delta_2},
    \qquad
    \underline\Delta_\varepsilon:=\min(\Delta_0,\gamma_1\overline\Delta_\varepsilon)
  \end{equation}
  We show by induction on $k$ that $\Delta_k\ge\underline\Delta_\varepsilon$ throughout this segment. The bound holds for $k=0$, and accepted steps do not reduce the radius. If a step is rejected, \eqref{eq:small-radius-accepted} and $G_k>\varepsilon$ imply $\Delta_k>\overline\Delta_\varepsilon$, hence $\Delta_{k+1}=\gamma_1\Delta_k>\gamma_1\overline\Delta_\varepsilon\ge\underline\Delta_\varepsilon$.

  For $0<\varepsilon\le\varepsilon_0:=\delta_2\Delta_0/(\gamma_1\alpha)$, this lower bound is $\underline\Delta_\varepsilon=\gamma_1\alpha\varepsilon/\delta_2$. Each accepted iteration before termination therefore satisfies
  \begin{align}
    \Fcal(\rho_k)-\Fcal(\rho_{k+1})
     & =\varrho_kp_k \notag           \\
     & \ge\frac{\eta_1}{2}\varepsilon
    \min\!\left(\frac{\varepsilon}{\delta_2},\underline\Delta_\varepsilon\right)
    =c_\ast\frac{\varepsilon^2}{\delta_2},
    \label{eq:successful-decrease}
  \end{align}
  where $c_\ast:=\eta_1\gamma_1\alpha/2$. Since $\Fcal\ge\underline V+\tfrac12\underline W$, the number $S_\varepsilon$ of accepted iterations before $K_\varepsilon$ obeys
  \begin{equation}
    \label{eq:successful-count}
    S_\varepsilon\le
    \frac{(\Fcal(\rho_0)-\inf\Fcal)\delta_2}
    {c_\ast\varepsilon^2}
  \end{equation}

  We now bound the number $U_\varepsilon$ of rejected iterations before $K_\varepsilon$. A rejection multiplies the radius by $\gamma_1$, whereas any accepted iteration multiplies it by at most $\gamma_2$. The multiplicative radius updates together with \eqref{eq:radius-floor} give
  \[
    \underline\Delta_\varepsilon
    \le\Delta_{K_\varepsilon}
    \le\Delta_0\gamma_1^{U_\varepsilon}\gamma_2^{S_\varepsilon}
  \]
  and hence
  \begin{equation}
    \label{eq:rejected-count}
    U_\varepsilon\le
    \frac{S_\varepsilon\log\gamma_2+\log(\Delta_0/\underline\Delta_\varepsilon)}
    {|\log\gamma_1|}
  \end{equation}
  Combining \eqref{eq:successful-count}--\eqref{eq:rejected-count},
  \begin{align}
    K_\varepsilon
    \le{} & \left(1+\frac{\log\gamma_2}{|\log\gamma_1|}\right)
    \frac{(\Fcal(\rho_0)-\inf\Fcal)\delta_2}
    {c_\ast\varepsilon^2} \notag                                     \\
          & +\frac{\log(\Delta_0/\underline\Delta_\varepsilon)}{|\log\gamma_1|}
    =O(\varepsilon^{-2})
    \label{eq:total-count}
  \end{align}
  If no iterate reached the tolerance, the same counting would apply to arbitrarily long initial segments and bound their length by the right-hand side of \eqref{eq:total-count}, which is impossible. Thus $K_\varepsilon<\infty$ for every sufficiently small $\varepsilon$, and hence $\liminf_kG_k=0$. Accepted iterations decrease $\Fcal$; rejected iterations leave $\rho_k$ unchanged.
\end{proof}

\subsection{Convergence of the gradient norms}
\label{sec:lim-upgrade}

Theorem~\ref{thm:global} shows only that the gradient norms become small infinitely often. To show that they tend to zero, we must compare the gradient norms of consecutive iterates. These gradients belong to different spaces, $L^2(\rho_k)$ and $L^2(\rho_{k+1})$. Composing the new gradient with the accepted pushforward map brings it back to $L^2(\rho_k)$, where the two can be compared directly.

\begin{lemma}[Comparison of gradients at consecutive iterates]
  \label{lem:cross-iterate}
  Let $\rho\in\Pcal$, $s\in L^2(\rho;\R^d)$, $T:=\id+s$, and $\rho':=T_\#\rho=R_\rho(s)$. Then, with $L_g:=C_V+\sqrt2\,C_W$,
  \[
    \Bigl|\,\normrho{\gradW\Fcal(\rho')}{\rho'}-\normrho{\gradW\Fcal(\rho)}{\rho}\,\Bigr|
    \;\le\;\bigl\|\gradW\Fcal(\rho')\circ T-\gradW\Fcal(\rho)\bigr\|_{L^2(\rho)}
    \;\le\;L_g\,\normrho{s}{\rho}
  \]
\end{lemma}

\begin{proof}
  Write $g_\mu:=\gradW\Fcal(\mu)=\nabla V+\nabla W*\mu$. The Hessian bounds give linear growth, so $g_\mu\in L^2(\mu)$ for every $\mu\in\Pcal$. Moreover,
  \[
    \int |x|^2\,d\rho'=\int |x+s(x)|^2\,d\rho\le 2\int |x|^2\,d\rho+2\normrho{s}{\rho}^2<\infty
  \]
  and therefore $\rho'=T_\#\rho\in\Pcal$. The change-of-variables identity for pushforwards applied to the nonnegative Borel function $|g_{\rho'}|^2$ gives
  \[
    \normrho{g_{\rho'}}{\rho'}^2=\int|g_{\rho'}|^2\,d\rho'=\int|g_{\rho'}\circ T|^2\,d\rho
  \]
  The first inequality is the reverse triangle inequality in $L^2(\rho)$. For the second, decompose pointwise:
  \[
    g_{\rho'}(T(x))-g_\rho(x)
    =\bigl[\nabla V(T(x))-\nabla V(x)\bigr]
    +\bigl[(\nabla W*\rho')(T(x))-(\nabla W*\rho)(x)\bigr]
  \]
  The potential term is bounded by $C_V|s(x)|$, since $\nabla V$ is $C_V$-Lipschitz; its $L^2(\rho)$ norm is at most $C_V\normrho{s}{\rho}$. For the interaction term, $\rho'=T_\#\rho$ lets us write the convolution with $\rho'$ as an integral against $\rho$:
  \[
    (\nabla W*\rho')(T(x))-(\nabla W*\rho)(x)
    =\int\Bigl(\nabla W\bigl(x-x'+s(x)-s(x')\bigr)-\nabla W(x-x')\Bigr)d\rho(x')
  \]
  The norm of this term is pointwise bounded by $C_W\int|s(x)-s(x')|\,d\rho(x')$, since $\nabla W$ is $C_W$-Lipschitz. Jensen's inequality in the inner integral, Fubini, and the variance identity from the proof of Proposition~\ref{prop:bddhess} give
  \[
    \Bigl\|\,C_W\!\int|s(\cdot)-s(x')|\,d\rho(x')\Bigr\|_{L^2(\rho)}^2
    \le C_W^2\iint|s(x)-s(x')|^2\,d\rho(x')\,d\rho(x)
    \le2C_W^2\normrho{s}{\rho}^2
  \]
  The triangle inequality then gives $L_g\normrho{s}{\rho}$.
\end{proof}

\begin{theorem}[Gradient convergence]
  \label{thm:lim}
  In the setting of Theorem~\ref{thm:global},
  \[
    \lim_{k\to\infty}\,\normrho{\gradW\Fcal(\rho_k)}{\rho_k}=0
  \]
\end{theorem}

\begin{proof}
  By Theorem~\ref{thm:global}, $\liminf_kG_k=0$, $\Fcal(\rho_k)$ is nonincreasing, and $\inf\Fcal>-\infty$. Suppose, for contradiction, that $\limsup_kG_k>2\varepsilon$ for some $\varepsilon>0$.

  \emph{Segments.} Construct $m_1<l_1<m_2<l_2<\cdots$ inductively: $m_j$ is the first index after $l_{j-1}$ (with $l_0:=0$) with $G_{m_j}\ge2\varepsilon$, which exists since $G_k\ge2\varepsilon$ infinitely often; $l_j$ is the first index after $m_j$ with $G_{l_j}<\varepsilon$, which exists since $\liminf_kG_k=0$. By minimality, $G_k\ge\varepsilon$ on $[m_j,l_j)$, and the segments are pairwise disjoint.

  \emph{Decrease along a segment.} Let $A_j$ be the accepted indices in $[m_j,l_j)$. On the segment, $G_k\ge\varepsilon$ and \eqref{eq:cauchy-detailed} give $p_k>0$, so $\varrho_k$ is well defined. Since every accepted step has $\varrho_k\ge\eta_1>0$, each $k\in A_j$ satisfies
  \[
    \Fcal(\rho_k)-\Fcal(\rho_{k+1})=\varrho_kp_k
    \ge\frac{\eta_1\varepsilon}{2}\min\Bigl(\frac{\varepsilon}{\delta_2},\Delta_k\Bigr)
  \]
  Rejected iterations change neither $\rho_k$ nor $\Fcal(\rho_k)$. Summing over the segment and writing $D_j:=\Fcal(\rho_{m_j})-\Fcal(\rho_{l_j})\ge0$, monotonicity, boundedness below, and disjointness give $\sum_jD_j<\infty$, hence $D_j\to0$, hence
  \[
    \Sigma_j:=\sum_{k\in A_j}\min\Bigl(\frac{\varepsilon}{\delta_2},\Delta_k\Bigr)
    \le\frac{2D_j}{\eta_1\varepsilon}\longrightarrow0
  \]

  \emph{Sum of the radii of accepted iterations.} Once $\Sigma_j<\varepsilon/\delta_2$, every $k\in A_j$ has $\min(\varepsilon/\delta_2,\Delta_k)=\Delta_k$, so $\sum_{k\in A_j}\Delta_k=\Sigma_j$ tends to zero.

  \emph{Gradient change along a segment.} Between $m_j$ and $l_j$ the measure changes only at accepted indices, where $\rho_{k+1}=R_{\rho_k}(s_k)$ with $\normrho{s_k}{\rho_k}\le\Delta_k$. Lemma~\ref{lem:cross-iterate} at each accepted index and $G_{k+1}=G_k$ at each rejected index give
  \[
    \varepsilon\le G_{m_j}-G_{l_j}
    \le\sum_{k=m_j}^{l_j-1}|G_{k+1}-G_k|
    \le L_g\!\sum_{k\in A_j}\normrho{s_k}{\rho_k}
    \le L_g\!\sum_{k\in A_j}\Delta_k
    \longrightarrow0
  \]
  a contradiction. Hence no $\varepsilon>0$ satisfies $\limsup_kG_k>2\varepsilon$; that is, $\lim_kG_k=0$.
\end{proof}

\begin{corollary}[Stationarity of accumulation points]
  \label{cor:accum}
  Assume in addition quadratic coercivity, $V(x)\ge c_0|x|^2-c_1$ with $c_0>0$. Then $(\rho_k)$ is tight. Every weak accumulation point $\bar\rho$ satisfies $\nabla V+\nabla W*\bar\rho=0$ $\bar\rho$-almost everywhere; that is, $\bar\rho$ is stationary.
\end{corollary}

\begin{proof}
  Since $\Fcal(\rho_k)\le\Fcal(\rho_0)$ and $\Wcal\ge\tfrac12\underline W$, quadratic coercivity gives
  \[
    \int |x|^2\,d\rho_k(x)\le c_0^{-1}\!\left(\Fcal(\rho_0)-\tfrac12\underline W+c_1\right)=:m_2
  \]
  uniformly. Therefore Markov's inequality gives $\sup_k\rho_k(\{|x|>R\})\le m_2/R^2\to0$ as $R\to\infty$, so the family is tight. By Prokhorov's theorem, weak accumulation points therefore exist. Let $\rho_{k_j}\rightharpoonup\bar\rho$ be any weakly convergent subsequence; by Theorem~\ref{thm:lim}, $\|g_{k_j}\|_{\rho_{k_j}}\to0$ along it. Write $G_j(x):=\nabla V(x)+\nabla W*\rho_{k_j}(x)$ and $\bar G:=\nabla V+\nabla W*\bar\rho$. Since $\nabla W$ is globally Lipschitz ($\|\nabla^2W\|_\infty\le C_W$), weak convergence together with the uniform second-moment bound gives $\nabla W*\rho_{k_j}\to\nabla W*\bar\rho$ locally uniformly by Lemma~\ref{lem:convolution-continuity} in Appendix~\ref{app:convolution}. Hence $|G_j|^2\to|\bar G|^2$ locally uniformly, and $|\bar G(x)|^2$ is continuous with quadratic growth. Choose continuous compactly supported cutoffs with $0\le\chi_R\le1$ and $\chi_R\uparrow1$. Then
  \[
    \int\chi_R\,|\bar G|^2\,d\bar\rho=\lim_j\int\chi_R\,|G_j|^2\,d\rho_{k_j}\ \le\ \liminf_j\int|G_j|^2\,d\rho_{k_j}=\lim_j\|g_{k_j}\|_{\rho_{k_j}}^2=0
  \]
  The first equality holds because $\chi_R|G_j|^2\to\chi_R|\bar G|^2$ uniformly on the compact support of $\chi_R$ and $\chi_R|\bar G|^2$ is bounded and continuous, so weak convergence applies; monotone convergence then gives
  \[
    \int|\bar G|^2\,d\bar\rho
    =\lim_{R\to\infty}\int\chi_R|\bar G|^2\,d\bar\rho=0
  \]
  and therefore $\bar G=0$ $\bar\rho$-almost everywhere.
\end{proof}

\subsection{A tangent-space variant}
\label{sec:tangent-variant}

Let $P_\rho$ be the orthogonal projection from $L^2(\rho;\R^d)$ onto the closed space $T_\rho\Pcal$, and write $g_\rho=\nabla V+\nabla W*\rho$. Define the tangent-space operator $\mathsf H_\rho$, obtained by restricting $\Bop_\rho$ to $T_\rho\Pcal$ and projecting its output back onto $T_\rho\Pcal$, and the corresponding model by
\begin{equation}
  \label{eq:tangent-model}
  \mathsf H_\rho:=P_\rho\Bop_\rho\big|_{T_\rho\Pcal},\qquad
  m_\rho^T(s):=\Fcal(\rho)+\iprho{g_\rho}{s}{\rho}
  +\frac12\iprho{s}{\mathsf H_\rho s}{\rho},\quad s\in T_\rho\Pcal
\end{equation}

\begin{corollary}[Trust-region steps in the tangent space]
  \label{cor:tangent-tr}
  Under Assumption~\ref{ass:smooth-class}, run Algorithm~\ref{alg:steihaug} in $T_{\rho_k}\Pcal$ for $m_{\rho_k}^T$, computing products with $\mathsf H_{\rho_k}$ instead of $\Bop_{\rho_k}$, CG-iteration limit $J\in\{1,2,\ldots\}$, and tolerance $\omega\in[0,1)$. Use the pushforward, ratio test, stopping rule, and acceptance and radius-update parameters from \S\ref{sec:method}. Starting from any $\rho_0\in\Pcal$, the objective is nonincreasing and $\|g_{\rho_k}\|_{\rho_k}\to0$. For sufficiently small $\varepsilon>0$, the index of the first $\varepsilon$-stationary iterate satisfies \eqref{eq:total-count}. Under quadratic coercivity, the iterates are tight and every weak accumulation point is stationary.
\end{corollary}

\begin{proof}
  First, $g_\rho\in T_\rho\Pcal$. For $\rho\in\Pcal$, set $U_\rho=V+W*\rho-(V+W*\rho)(0)$. Differentiation under the convolution gives $U_\rho\in C^2$, $\nabla U_\rho=g_\rho$, and $\|\nabla^2U_\rho\|_\infty\le\delta_2$. The Hessian bound gives $|g_\rho(x)|\le |g_\rho(0)|+\delta_2|x|$ and $|U_\rho(x)|\le |g_\rho(0)|\,|x|+\tfrac12\delta_2|x|^2$. Thus $g_\rho$ has at most linear growth and $U_\rho$ at most quadratic growth. Choose smooth cutoffs $\chi_R$ equal to one on $B_R$, supported in $B_{2R}$, with $|\nabla\chi_R|\le C/R$. For $R\ge1$ these growth bounds give
  \[
    |\nabla(\chi_RU_\rho)(x)-g_\rho(x)|
    \le C_\rho(1+|x|)\mathbf1_{\{|x|\ge R\}}
  \]
  with $C_\rho$ independent of $R$. The right-hand side tends to zero in $L^2(\rho)$ as $R\to\infty$, so $g_\rho$ is an $L^2(\rho)$ limit of gradients of compactly supported $C^2$ functions. Each such gradient is in turn a uniform limit of gradients of $C_c^\infty$ functions, obtained by mollification. Choosing the mollification parameter along $R\to\infty$ gives a sequence in $\{\nabla\phi:\phi\in C_c^\infty\}$ converging to $g_\rho$ in $L^2(\rho)$, hence $g_\rho\in T_\rho\Pcal$.

  For $u,v\in T_\rho\Pcal$, orthogonality gives $\iprho{u}{\mathsf H_\rho v}{\rho}=\iprho{u}{\Bop_\rho v}{\rho}$. Self-adjointness of $\Bop_\rho$ and contractivity of $P_\rho$ therefore imply that $\mathsf H_\rho$ is self-adjoint and
  \[
    \|\mathsf H_\rho\|\le\delta_2,\qquad
    \iprho{s}{\mathsf H_\rho s}{\rho}=\iprho{s}{\Bop_\rho s}{\rho},\qquad
    m_\rho^T(s)=m_\rho(s)\quad(s\in T_\rho\Pcal)
  \]
  The subsolver and Proposition~\ref{prop:steihaug-cauchy} apply in this Hilbert space. Proposition~\ref{prop:bddhess} then gives, for $G_k:=\|g_{\rho_k}\|_{\rho_k}>0$,
  \[
    p_k\ge\frac12G_k\min\!\left\{\frac{G_k}{\delta_2},\Delta_k\right\},\qquad
    |\Fcal(R_{\rho_k}(s_k))-m_{\rho_k}^T(s_k)|\le\delta_2\|s_k\|_{\rho_k}^2
  \]
  These estimates reproduce the radius and iteration-count bounds of Theorem~\ref{thm:global}. Lemma~\ref{lem:cross-iterate} also applies to tangent steps, so the same proof gives \(G_k\to0\) for this variant. Quadratic coercivity then yields Corollary~\ref{cor:accum}.
\end{proof}

\begin{remark}[Finite-support measures]
  \label{rem:atomic-tangent}
  For $\rho=\sum_{i=1}^Nw_i\delta_{x_i}$ with distinct support points and positive weights summing to one, $T_\rho\Pcal=L^2(\rho;\R^d)$. Given vectors $v_i$, choose smooth cutoffs $\psi_i$ with disjoint compact supports and $\psi_i=1$ near $x_i$. Then $\phi(x)=\sum_{i=1}^N\psi_i(x)\,v_i\cdot(x-x_i)$ belongs to $C_c^\infty$ and satisfies $\nabla\phi(x_i)=v_i$. Thus $P_\rho=I$ and $\mathsf H_\rho=\Bop_\rho$. Pushforwards preserve finite support, so from a finite-support initial measure, Algorithm~\ref{alg:wtr} and its tangent-space variant in Corollary~\ref{cor:tangent-tr} produce the same iterates.
\end{remark}

\Needspace{6\baselineskip}
\section{Numerical experiments}
\label{sec:numerics}

\subsection{The method for empirical measures}
\label{sec:empirical}

For an empirical measure $\rho_N=N^{-1}\sum_{i=1}^N\delta_{X_i}$, the particle objective and metric are
\[
  \Fcal_N(X)=\frac1N\sum_{i=1}^N V(X_i)
  +\frac1{2N^2}\sum_{i,j=1}^N W(X_i-X_j),
  \qquad
  \|S\|_{\rho_N}^2=\frac1N\sum_{i=1}^N|S_i|^2
\]
For distinct particle locations, a field in $L^2(\rho_N;\mathbb R^d)$ is determined by its values $S=(S_1,\ldots,S_N)$, and Algorithm~\ref{alg:wtr} applied to $\rho_N$ is an iteration on the particle positions $X=(X_1,\ldots,X_N)$ with the norm $\|\cdot\|_{\rho_N}$. The operator bound $\|\Bop_\rho\|\le\delta_2=C_V+C_W$, the model-error bound $\delta_2\|s\|_\rho^2$, and the Cauchy decrease \eqref{eq:cauchy-detailed} hold with the same $\delta_2$ for every $N$. Consequently, the bound in Theorem~\ref{thm:global} is independent of $N$ provided $\sup_N\bigl(\Fcal_N(X_0^{(N)})-\inf_X\Fcal_N(X)\bigr)<\infty$ and the same trust-region parameters are used for every $N$.
For $\rho_N$, the Wasserstein-gradient norm is
\[
  \|\gradW\Fcal(\rho_N)\|_{\rho_N}=\Bigl(\frac1N\sum_{i=1}^N\Bigl|\nabla V(X_i)+\frac1N\sum_{j=1}^N\nabla W(X_i-X_j)\Bigr|^2\Bigr)^{1/2}
\]
which is the root-mean-square force on the particles. We call it the gradient norm; in each experiment below, all solvers stop when it falls below the same gradient tolerance.

\subsection{The Gaussian core model in a harmonic trap}
\label{sec:gaussian-core}

We return to the soft-matter model of the introduction: the centers of
mass of flexible dendrimers repel each other through the pair interaction of the Gaussian core
model
\cite{LouisBolhuisHansen2000,LouisBolhuisHansenMeijer2000,Likos2002}
and are confined by an isotropic harmonic trap, which approximates an
optical trap near its center \cite{NeumanBlock2004}.

The pair interaction of the Gaussian core model and the isotropic
harmonic trap potential are
\begin{equation}
  \label{eq:dendrimer-energy}
  W(z)=a\exp\!\left(-\frac{|z|^2}{2\sigma^2}\right),
  \qquad
  V(x)=\frac{\kappa}{2}|x|^2,
  \qquad
  a,\sigma,\kappa>0
\end{equation}
where $a$ is the energy of complete overlap, $\sigma$ is the
effective particle size, and $\kappa$ is the trap stiffness. We work in reduced units with
$a=1$ and $\sigma=1/\sqrt2$.
The corresponding mean-field energy is
\begin{equation}
  \label{eq:dendrimer-functional}
  \mathcal F(\rho)
  =
  \frac{\kappa}{2}
  \int |x|^2\,d\rho(x)
  +
  \frac{a}{2}
  \iint
  \exp\!\left(
    -\frac{|x-y|^2}{2\sigma^2}
  \right)
  d\rho(x)d\rho(y)
\end{equation}

For the empirical measure
$\rho_N=N^{-1}\sum_{i=1}^N\delta_{X_i}$, the particle energy is
\begin{equation}
  \label{eq:dendrimer-empirical-energy}
  \mathcal F_N(X)
  =
  \frac{\kappa}{2N}
  \sum_{i=1}^N |X_i|^2
  +
  \frac{a}{2N^2}
  \sum_{i,j=1}^N
  \exp\!\left(
    -\frac{|X_i-X_j|^2}{2\sigma^2}
  \right)
\end{equation}

The Gaussian interaction and the harmonic trap have globally bounded
Hessians; in particular, $\|\nabla^2 W\|_\infty\le a/\sigma^2$,
so this model belongs to the smooth potential--interaction class
covered by the convergence analysis.

\paragraph{Time to the gradient tolerance}
We minimize the two-dimensional model with $N\in\{2048,4096\}$
particles and trap stiffnesses $\kappa\in\{0.02,0.08,0.32\}$, starting
both solvers from the same random initial configurations. WTR uses one
fixed set of parameters, whereas L-BFGS is tuned separately for each
pair $(\kappa,N)$. At the gradient tolerance $2\times10^{-4}$, WTR is
faster than L-BFGS for every $\kappa$, $N$, and initial configuration,
with a median speedup of $1.40$.
Table~\ref{tab:gaussian-core-performance} reports the median times and
speedups for each pair $(\kappa,N)$, and
Figure~\ref{fig:gaussian-core-speedup} shows the speedup for each
initial configuration.

\begin{figure}[tbp]
  \centering
  \begin{minipage}[c]{0.54\textwidth}
    \centering
    \makeatletter\def\@captype{table}\makeatother
    \caption{Gaussian core model at gradient tolerance $2\times10^{-4}$:
    median time to tolerance and median speedup of WTR for each pair $(\kappa,N)$.}
    \label{tab:gaussian-core-performance}
    \resizebox{\linewidth}{!}{\begingroup
\small
\setlength{\tabcolsep}{3.5pt}
\begin{tabular}{rr rr r}
\hline
$\kappa$ & $N$ & \multicolumn{2}{c}{Time to tolerance (s)} & Speedup \\
 & & WTR & L-BFGS & L-BFGS/WTR \\
\hline
0.02 & 2048 & 0.220 & 0.307 & 1.37 \\
0.02 & 4096 & 0.933 & 1.219 & 1.29 \\
0.08 & 2048 & 0.244 & 0.504 & 2.16 \\
0.08 & 4096 & 0.958 & 1.971 & 2.06 \\
0.32 & 2048 & 0.292 & 0.368 & 1.28 \\
0.32 & 4096 & 1.076 & 1.461 & 1.33 \\
\hline
\end{tabular}
\endgroup
}
  \end{minipage}\hfill
  \begin{minipage}[c]{0.42\textwidth}
    \centering
    \includegraphics[width=\textwidth]{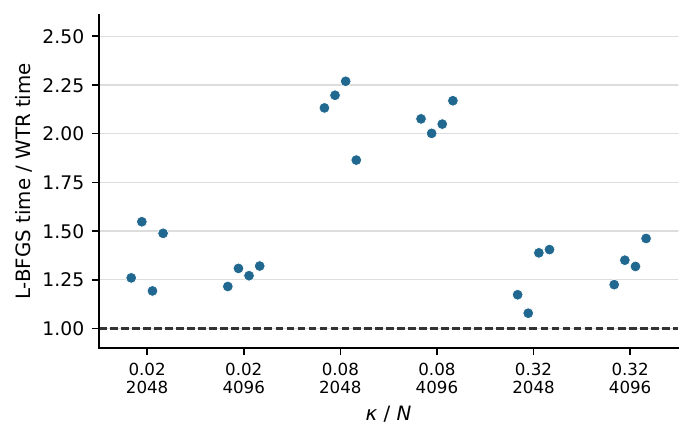}
    \caption{Speedup of WTR over L-BFGS for each trap stiffness $\kappa$,
    particle number $N$, and initial configuration. Values above one mean
    that WTR is faster.}
    \label{fig:gaussian-core-speedup}
  \end{minipage}
\end{figure}

Figure~\ref{fig:gaussian-core-dynamics} shows the configurations at
the gradient tolerance for $N=4096$. The top row shows the particle
positions when WTR stops, for one initial configuration. In the middle
row, the radial distributions of the final WTR and L-BFGS
configurations are indistinguishable at the scale of the plot.
The bottom row shows the gradient norm against time for both solvers
and all four initial configurations.

\paragraph{Structure of the computed configurations}
As $\kappa$ increases, the trap confines the particles to a smaller
region: the root-mean-square distance from the origin is $2.06$,
$1.38$, and $0.85$ for $\kappa=0.02$, $0.08$, and $0.32$, for both
solvers. The particles concentrate on rings. In
the WTR configurations of Figure~\ref{fig:gaussian-core-dynamics},
$65\%$ of the particles lie on a ring of radius $1.63$ and width
$0.05$ for $\kappa=0.08$. For $\kappa=0.32$, $97.1\%$ of the particles
lie on a single ring of radius $0.862$ and width $0.004$, and $2.8\%$
lie within $0.03$ of the origin. This state is close to a minimizer of
the mean-field energy \eqref{eq:dendrimer-functional} among measures
of the same form. For the measure with mass $m$ distributed uniformly
on the circle of radius $R$ and mass $1-m$ at the origin, the energy
equals
\[
  \frac{\kappa}{2}mR^2
  +\frac a2\Bigl[(1-m)^2+2m(1-m)e^{-R^2/(2\sigma^2)}
  +m^2e^{-R^2/\sigma^2}I_0\bigl(R^2/\sigma^2\bigr)\Bigr]
\]
where $I_0$ is the modified Bessel function of order zero. For
$\kappa=0.32$, this expression is minimized at $R=0.866$ and $m=0.965$;
the computed radius differs from this value by $0.4\%$, and the
computed mass on the ring by $0.006$. Concentration on rings and on points is known
for stationary states and minimizers of interaction energies
\cite{Kolokolnikov2011,Balague2013}.

\begin{figure}[tbp]
  \centering
  \includegraphics[width=0.72\textwidth]{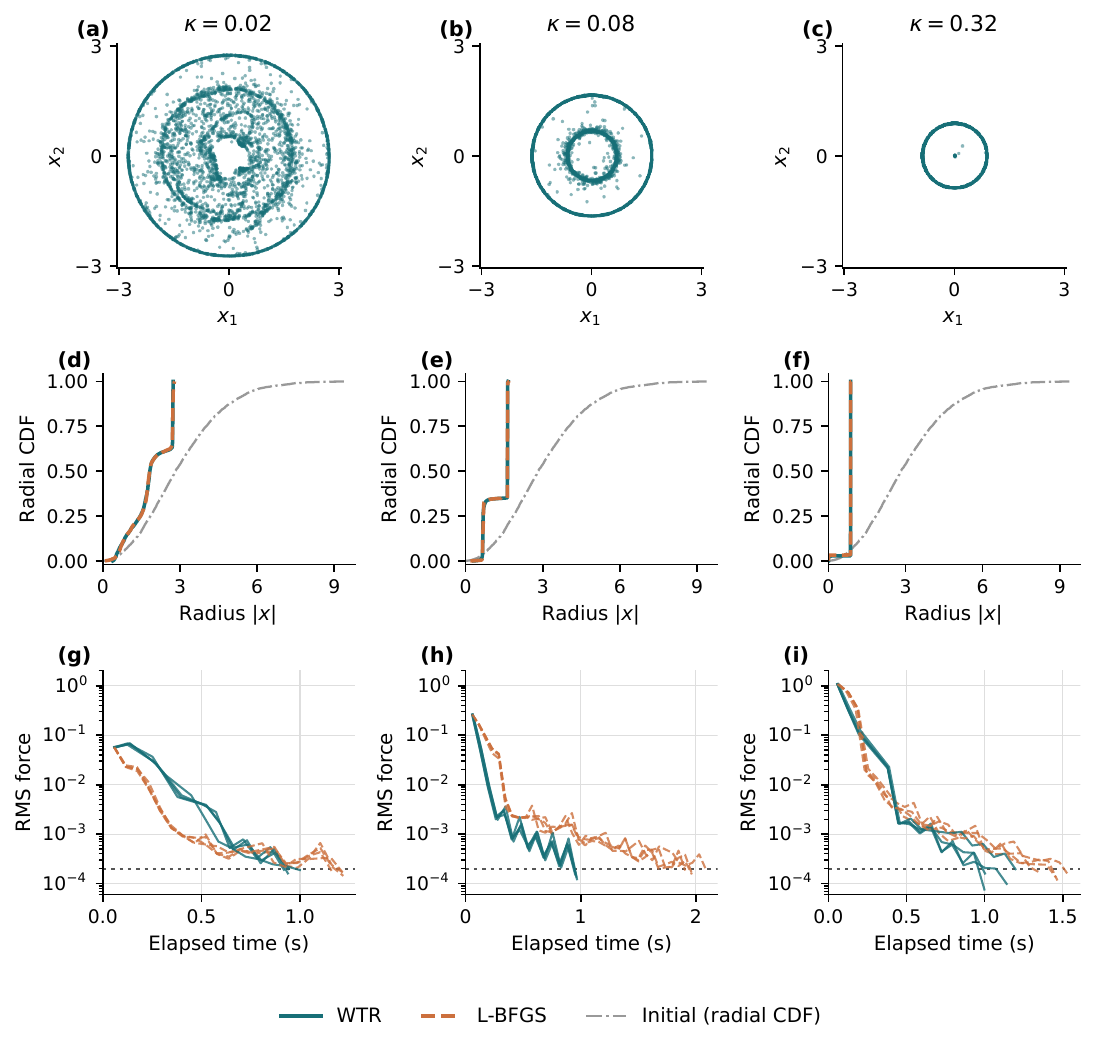}
  \caption{Configurations of the Gaussian core model and gradient-norm histories at $N=4096$,
  with columns corresponding to $\kappa=0.02,0.08,0.32$.
  Top: particle positions when WTR stops, for one of the four initial
  configurations.
  Middle: empirical cumulative distributions of the distances of the
  particles from the origin, for that initial configuration and for the
  final configurations of WTR and L-BFGS.
  Bottom: gradient norm, that is, the root-mean-square (RMS) force, of
  WTR and L-BFGS against elapsed time for the four initial
  configurations; the dotted line marks the tolerance $2\times10^{-4}$.}
  \label{fig:gaussian-core-dynamics}
\end{figure}
\subsection{Maximum-mean-discrepancy fitting: effect of the quadratic model and radius growth}
\label{sec:mmd}

Let $\pi_M=M^{-1}\sum_{m=1}^M\delta_{Y_m}$ be a fixed target sample and let $k(x,y)=\tilde k(x-y)$ be an even, positive-definite, translation-invariant kernel. As in \cite{Arbel2019}, the empirical squared MMD, including the diagonal terms $i=j$, satisfies
\begin{align}
  \frac12\operatorname{MMD}_k^2(\rho_N,\pi_M)
  ={} & -\frac1{NM}\sum_{i=1}^N\sum_{m=1}^M k(X_i-Y_m)
  +\frac1{2N^2}\sum_{i,j=1}^N k(X_i-X_j)+C_{\pi_M}
  \label{eq:empirical-mmd-energy}
\end{align}
where $C_{\pi_M}:=(2M^2)^{-1}\sum_{m,n=1}^M k(Y_m-Y_n)$ is independent of $X$. Thus, after dropping this constant, the particle-dependent loss has the potential--interaction form with $V_{\pi_M}(x)=-M^{-1}\sum_{m=1}^M k(x-Y_m)$ and $W(z)=k(z)$. We take $k$ to be an equal-weight mixture of four Gaussian kernels, so $V_{\pi_M}$ and $W$ have globally bounded Hessians.

We fit $N=1000$ particles to each of three non-Gaussian targets in two dimensions, the banana, two-moons, and ring distributions, and measure the fit by $\tfrac12\operatorname{MMD}_k^2$ against a separate holdout sample. To see what the quadratic model and radius growth contribute, we compare three variants: (i) the full method; (ii) a Cauchy-only variant that sets $B_k=0$, so that every step is the boundary step $s_k=-\Delta_k g_k/\|g_k\|_{\rho_k}$, with the same ratio test and radius updates; and (iii) a no-radius-growth variant, the full method with radius expansion disabled, so that the radius never increases.

The full and no-radius-growth variants reach the gradient tolerance, whereas the Cauchy-only variant reaches the iteration cap first. On all three targets, the median final holdout loss of the full method is $4.1$--$5.8\%$ lower than that of the Cauchy-only variant. Disabling radius growth changes the final loss by at most $0.09\%$ but costs $1$ to $6$ additional iterations. Figure~\ref{fig:mmd-targets} shows the fitted particles, the loss histories, and the final losses.

\begin{figure}[tbp]
  \centering
  \includegraphics[width=0.85\textwidth]{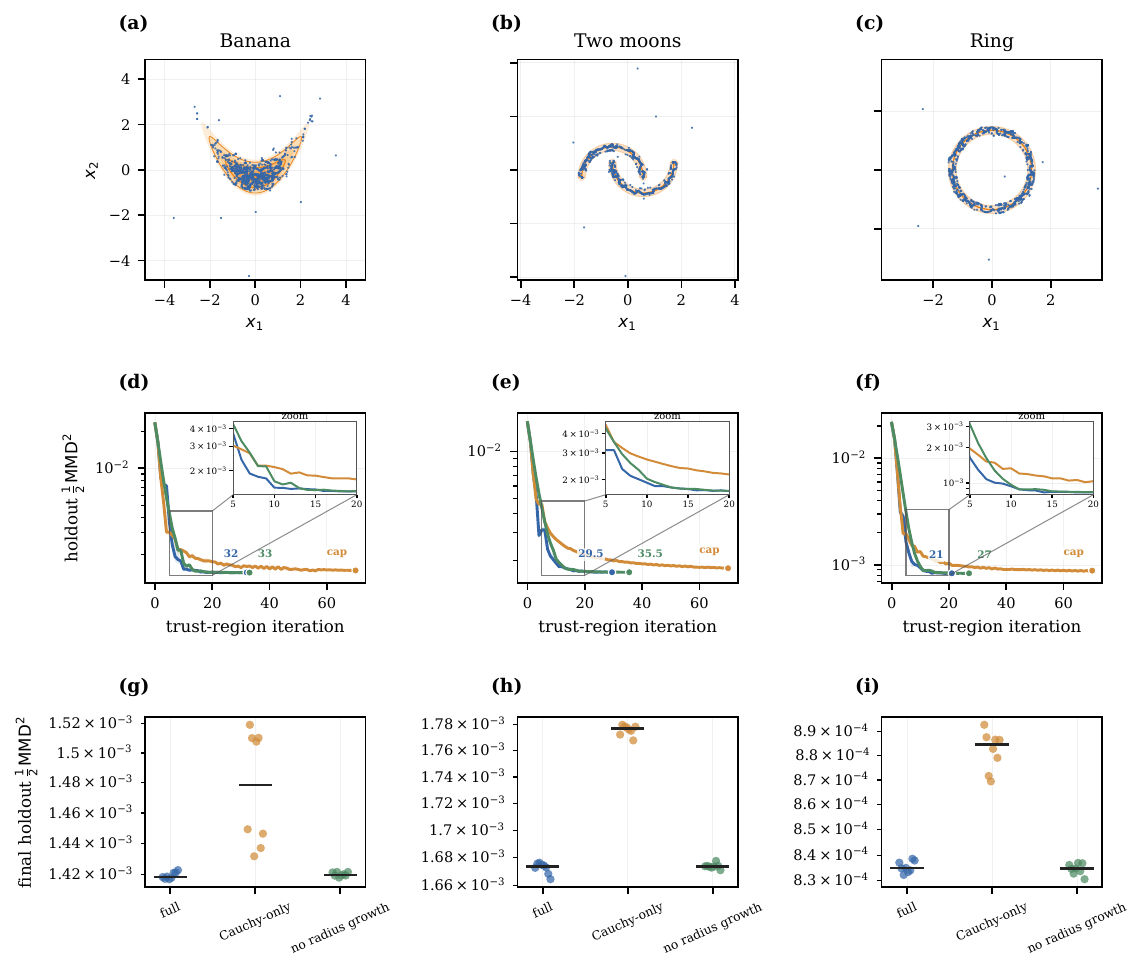}
  \caption{MMD fitting of the banana, two-moons, and ring targets.
    Blue, orange, and green denote the full, Cauchy-only, and
    no-radius-growth variants. (a)--(c) Final particles of the full
    method from one initial configuration (blue) and the target
    geometry (orange). (d)--(f) Median holdout loss with first and
    third quartiles over initial configurations, ending at the labeled
    median termination iteration; \emph{cap} marks a variant stopped by
    the iteration cap. Insets enlarge iterations 5--20. (g)--(i) Final
    holdout loss for each initial configuration; horizontal segments
    mark medians.}
  \label{fig:mmd-targets}
\end{figure}

\subsection{Equal-weight integration for non-Gaussian targets}
\label{sec:quadrature}

An equal-weight particle measure $\rho_N$ defines the integration rule
\[
  Q_Nf=\frac1N\sum_{i=1}^N f(X_i),\qquad\text{which approximates}\qquad
  I_\pi f=\int f\,d\pi
\]
We choose the nodes $X_1,\ldots,X_N$ by minimizing the empirical MMD
energy \eqref{eq:empirical-mmd-energy}, where $\pi_M$ is a large sample
from $\pi$. For $f$ in the reproducing kernel Hilbert space
$\mathcal H_k$ of the kernel, the reproducing property and the triangle
inequality give
\[
  |Q_Nf-I_\pi f|\le\|f\|_{\mathcal H_k}
  \bigl(\operatorname{MMD}_k(\rho_N,\pi_M)
       +\operatorname{MMD}_k(\pi_M,\pi)\bigr)
\]
so a small MMD gives accurate integration of every such $f$. We measure
accuracy by the root-mean-square error
\[
  E(X)=\left(\frac1{75}\sum_{\ell=1}^{75}
       |Q_Nf_\ell-I_\pi f_\ell|^2\right)^{1/2}
\]
where $f_\ell$ ranges over the $75$ Gaussian test functions
$f_{c,h}(x)=\exp\bigl(-|x-c|^2/(2h^2)\bigr)$ with centers
$c\in\{-1.8,-0.9,0,0.9,1.8\}^2$ and widths $h\in\{0.3,0.6,1.2\}$. We
report the construction time needed to reach a prescribed $E$. We build rules with
$N\in\{64,128,256\}$ nodes for the three targets of
Section~\ref{sec:mmd} and compare WTR with L-BFGS of memory $20$ and
$50$ (L-BFGS20 and L-BFGS50) and with SciPy's Krylov trust-region
solver \texttt{trust-krylov} (TR-Krylov) \cite{Virtanen2020}.

\paragraph{Time to integration accuracy}
Figure~\ref{fig:integration-performance} shows the error trajectories
and the time to the gradient tolerance.
Table~\ref{tab:integration-time} reports the time to reach a given
integration error at $N=128$. On the banana target, every solver levels
off near $E\approx6\times10^{-4}$, the error of the MMD
minimizer itself, so the table uses the threshold $10^{-3}$ there. On
the two-moons and ring targets, the median speedups of WTR over L-BFGS20, L-BFGS50, and TR-Krylov at
$E=10^{-4}$ are $1.83$, $1.78$, and $1.34$, respectively.
Figure~\ref{fig:integration-thresholds} shows how the speedups depend on
the threshold. At the loose threshold $10^{-2}$, WTR and both L-BFGS variants are about
equally fast (median speedups $0.96$); the speedup over L-BFGS grows as the
threshold is tightened and reaches about $1.8$ at $10^{-4}$.

\begin{figure}[tbp]
  \centering
  \includegraphics[width=0.82\textwidth]{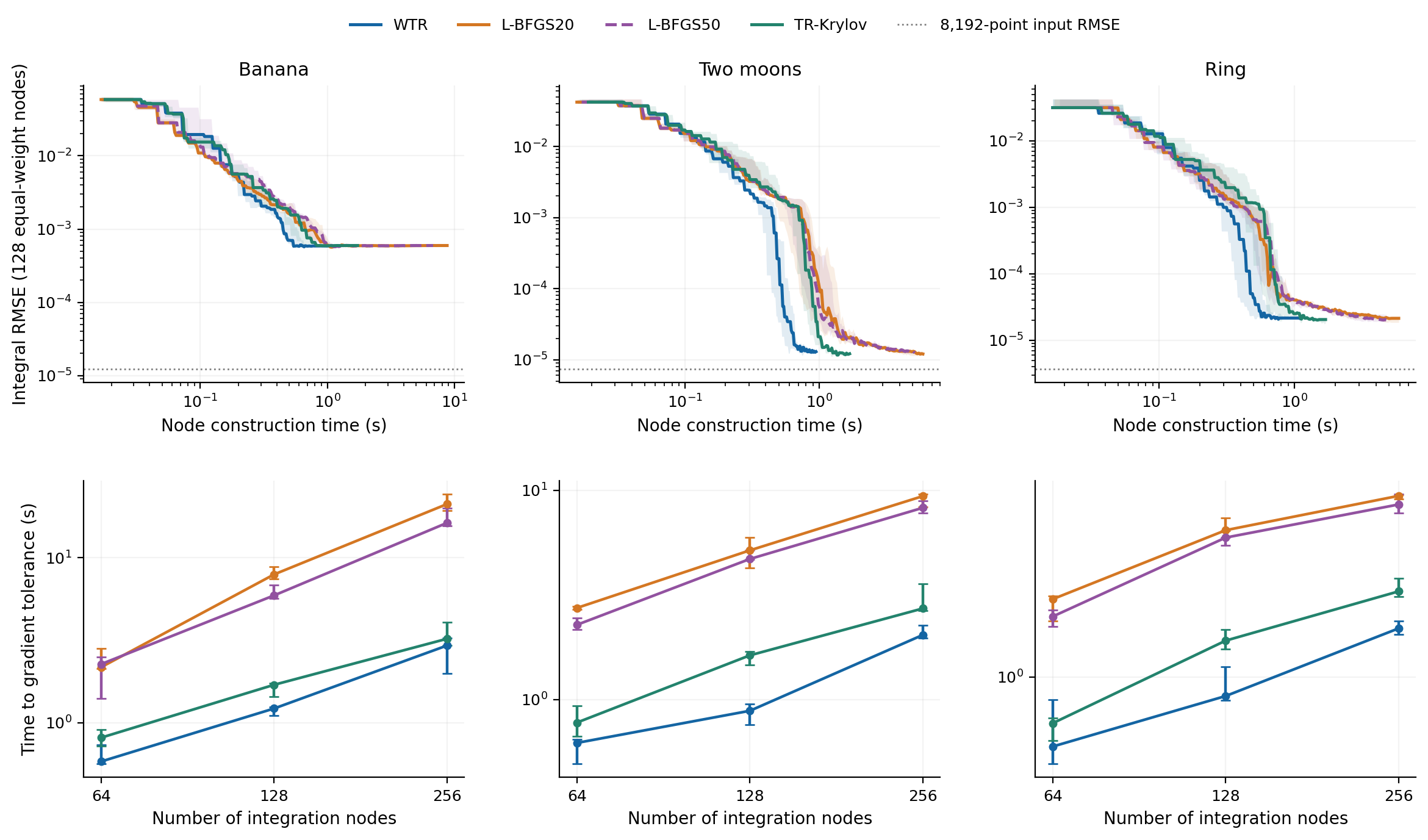}
  \caption{Equal-weight integration for three continuous targets.
  Top: integration RMSE against node construction time at $N=128$.
  Bottom: construction time to the gradient tolerance at each
  node count. Curves and markers are medians over three initial configurations,
  and bands and bars show the range.
  Dotted lines give the error of the equal-weight rule on all $8192$
  target sample points.}
  \label{fig:integration-performance}
\end{figure}

\begin{table}[tbp]
  \caption{Median construction time in seconds to the stated integration
  RMSE at $N=128$.}
  \label{tab:integration-time}
  \centering
  \begin{tabular}{lrrrrr}
\hline
Target & RMSE & WTR & L-BFGS20 & L-BFGS50 & TR-Krylov \\
\hline
Banana & $10^{-3}$ & 0.447 & 0.670 & 0.813 & 0.627 \\
Two moons & $10^{-4}$ & 0.529 & 1.018 & 0.956 & 0.888 \\
Ring & $10^{-4}$ & 0.473 & 0.635 & 0.699 & 0.715 \\
\hline
\end{tabular}

\end{table}

\begin{figure}[tbp]
  \centering
  \includegraphics[width=0.45\textwidth]{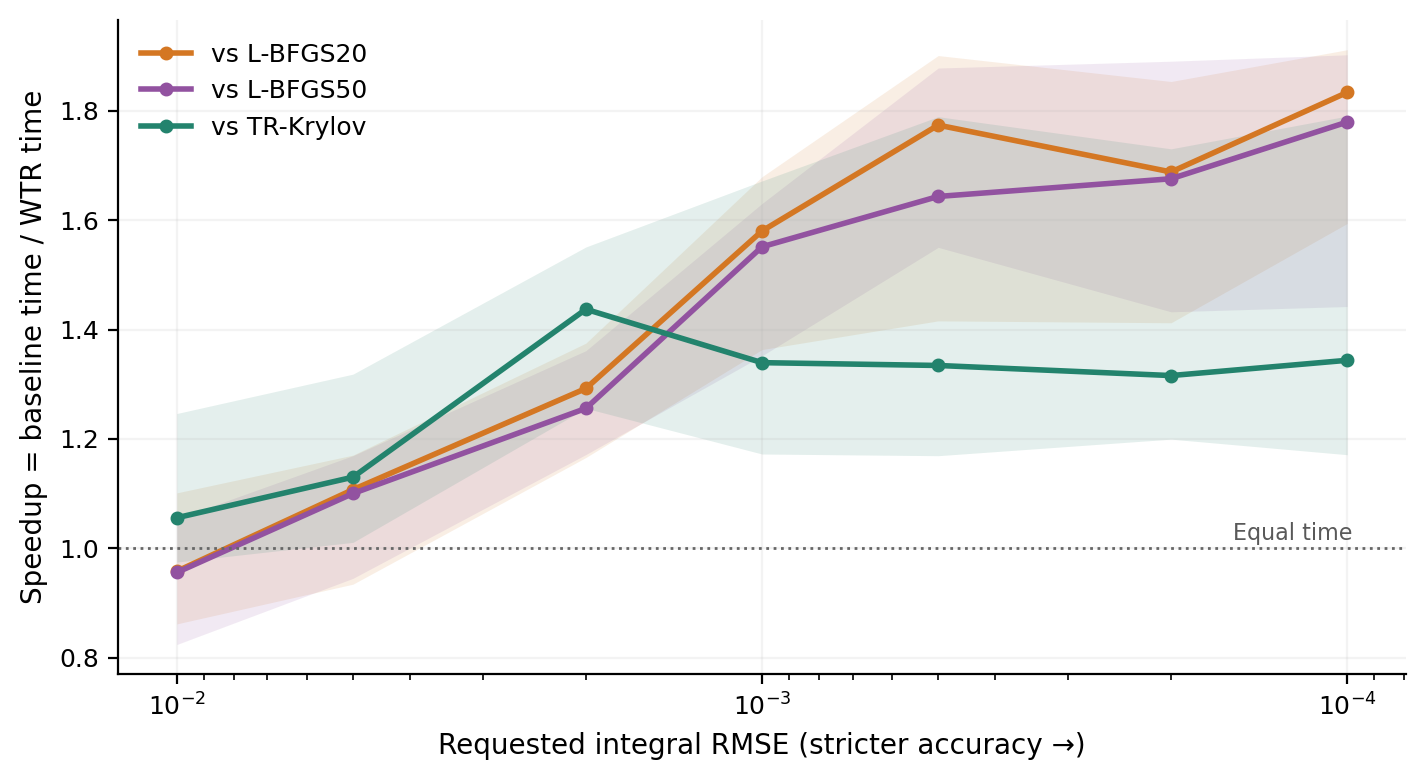}
  \caption{Median speedup of WTR at each integration threshold for the
  two-moons and ring targets with $N=64,128,256$; bands span the first
  and third quartiles. Values above one favor WTR.}
  \label{fig:integration-thresholds}
\end{figure}

If each solver is instead timed until it reaches the gradient
tolerance, the median speedups over all $27$ cases (three targets,
three node counts, and three initial configurations) are $4.98$,
$4.58$, and $1.40$ against L-BFGS20, L-BFGS50, and TR-Krylov. These
speedups are larger because, once the integration error stops
improving, the L-BFGS variants still take much longer than WTR to bring
the gradient norm down to the tolerance
(Figure~\ref{fig:integration-performance}, top row).
At the gradient tolerance, all four solvers produce rules of the same
accuracy. At $N=128$, for example, the median final banana errors lie
between $5.88\times10^{-4}$ and $5.96\times10^{-4}$, and the final
errors for two moons and ring are of order $10^{-5}$ for all four
methods. WTR reaches its final banana rule in $1.22$ seconds, compared
with $7.90$, $5.90$, and $1.70$ seconds for the three baselines.

\subsection{Dependence of the iteration count on the number of particles}
\label{sec:n-scan}

By Section~\ref{sec:empirical}, the constants of
Theorem~\ref{thm:global} are the same for every $N$. We report
iteration counts for $N$ from $10^2$ to $10^6$ on a two-dimensional
model whose potential has several wells,
\[
  V(x)=\frac{\kappa}{2}|x|^2-a_1\cos(v_1\cdot x)
  -\frac{a_2}{4}\cos(2v_2\cdot x),
  \qquad
  W(z)=b\bigl[\cos(u_1\cdot z)+\cos(u_2\cdot z)\bigr]
\]
with $\kappa=0.05$, $a_1=1$, $a_2=6$, and $b=0.3$, where $(v_1,v_2)$
and $(u_1,u_2)$ are the standard basis of $\mathbb R^2$ rotated by
$30^\circ$ and $45^\circ$. Assumption~\ref{ass:smooth-class} holds, and
the energy, the gradient, and a product with $\Bop_{\rho_N}$ can each
be evaluated in $O(N)$ operations. This energy has stationary measures
supported on finitely many points; we consider one with masses $0.5$,
$0.3$, and $0.2$ in three neighboring wells of $V$. WTR, with $J=20$,
$\omega=10^{-2}$, and the same parameters for every $N$, is compared
with L-BFGS with memory $10$ at the gradient tolerance $10^{-6}$. The
cost is measured by the number of evaluations of the energy, of the
gradient, and of products with $\Bop_{\rho_N}$.

When every particle is initially within distance $0.03$ of a point of
this stationary measure, WTR needs two or three iterations for every
$N$, none of them rejected, whereas L-BFGS needs $12$; the
corresponding numbers of evaluations are $15$ to $22$ and $28$. When
the initial positions are sampled uniformly from the disk of radius
$4.5$, which contains many wells, the median number of WTR iterations
grows from $18$ at $N=10^2$ to $43.5$ at $N=10^6$, and L-BFGS needs
about half as many evaluations as WTR ($106$ against $207$ at
$N=10^6$). In this example, WTR is therefore more efficient than
L-BFGS near the stationary measure and less efficient from a dispersed
initial configuration. The supplement gives the settings and the
results for every $N$.

\section{Conclusions and future directions}
\label{sec:discussion}

We have developed a trust-region method for smooth, possibly
nonconvex potential--interaction energies in Wasserstein space.
The method uses quadratic models along pushforward curves and
applies to every probability measure with finite second moment,
including empirical measures.
We establish monotone decrease of the
energy, convergence of the Wasserstein-gradient norms
to zero, and an $O(\varepsilon^{-2})$ bound on the total number
of iterations.
With an additional quadratic-coercivity assumption on the
potential, we also prove stationarity of every weak
accumulation point.
For empirical measures, the constants in the complexity bound
do not depend on the particle number.

We tested the method on two applications: a Gaussian core model of dendrimers in a harmonic trap, and MMD approximation of non-Gaussian
targets. For the two-dimensional Gaussian core model with $2048$ and
$4096$ particles, WTR reached the gradient tolerance faster than L-BFGS, with a median speedup of $1.40$. In the
MMD experiments, the quadratic model lowered the final holdout loss by
$4$--$6\%$, whereas disabling radius growth left the final loss essentially unchanged but increased the median number of iterations by $1$ to $6$.
When the method is used to build equal-weight integration rules for
the two-moons and ring targets, it reached an RMSE of $10^{-4}$ on our
Gaussian test functions about
$1.8$ times faster than L-BFGS and $1.3$ times faster than a Krylov
trust-region solver.
In the Gaussian core model at the largest trap stiffness, the computed
configuration is a ring containing $97\%$ of the particles with the rest
at the center, in agreement with the minimizer of the mean-field energy
among such measures. Near a stationary
measure supported on three points of a multi-well model, WTR needed two
or three iterations for every particle number from $10^2$ to $10^6$.

Several questions remain open.
On the theoretical side, one is to relate local convergence
rates to the curvature of the energy near stationary measures; the
iteration counts of Section~\ref{sec:n-scan} suggest a Newton-type
rate near measures supported on finitely many points, with constants
independent of the particle number.
Another is to bound the distance between the iterates for an empirical
measure and the iterates for the measure it approximates, which would guide
the choice of particle number for a given accuracy.
On the computational side, trust-region control could be
combined with inexact or stochastic evaluations and fast
interaction summation; the question is how to tie evaluation
accuracy to the radius while preserving sufficient decrease.
For distribution approximation, the joint effect of kernel
choice, target-sample size, and optimization accuracy remains
to be understood.

\section*{Acknowledgments}
This work was supported by the National Natural Science Foundation of China (12401233), NSFC International Creative Research Team (W2541005), National Key R$\&$D Program of China (2021ZD0201300), the Guangdong Provincial Key Laboratory of Mathematical and Neural Dynamical Systems (2024B1212010004), the Natural Science Foundation of Guangdong (Grant No. 2025A1515011188), the Guangdong-Dongguan Joint Research Fund (2023A1515140016), the Cross Disciplinary Research Team on Data Science and Intelligent Medicine (2023KCXTD054), the Hubei Provincial Natural Science Foundation of China (2026AFA034), Dongguan Key Laboratory for Data Science and Intelligent
Medicine and Hubei Key Laboratory of Engineering Modeling and Scientific Computing.

\section*{Data and code availability}
All data in Section~\ref{sec:numerics} are synthetic. The code that generates the data and the numerical results, figures, and tables of that section is available at \url{https://github.com/wan-you/wasserstein-trust-region}, together with dependency information and instructions for reproducing the experiments.

\section*{Declaration of competing interest}
The authors declare that they have no known competing financial interests or personal relationships that could have appeared to influence the work reported in this paper.

\begin{appendices}

  \section{Continuity of interaction fields under weak convergence}
  \label{app:convolution}

  Corollary~\ref{cor:accum} uses the following lemma with $K=\nabla W$. A globally Lipschitz kernel grows at most linearly, and the uniform second-moment bound controls the resulting tail integrals.

  \begin{lemma}
    \label{lem:convolution-continuity}
    Let $\mu_j,\mu\in\Pcal(\R^d)$ satisfy $\mu_j\rightharpoonup\mu$ weakly and $\sup_j\int|y|^2\,d\mu_j(y)\le M_2<\infty$. Let $K:\R^d\to\R^d$ be globally Lipschitz. Then
    \[
      (K*\mu_j)(x):=\int K(x-y)\,d\mu_j(y)
      \longrightarrow
      (K*\mu)(x)
    \]
    locally uniformly in $x$.
  \end{lemma}

  \begin{proof}
    Let $L$ be a global Lipschitz constant of $K$, so that $|K(z)|\le|K(0)|+L|z|$. Fix $x\in\R^d$. The map $y\mapsto K(x-y)$ is continuous, and each of its components is bounded by $|K(0)|+L|x|+L|y|$. The uniform second-moment bound gives, for every $R>0$ and every $j$,
    \[
      \int_{|y|>R}|y|\,d\mu_j(y)\le\frac{M_2}{R}
    \]
    so $\{\mu_j\}$ has uniformly integrable first moments. The last assertion of \cite[Lemma~5.1.7]{AGS2008} extends weak convergence to continuous test functions of linear growth in this case, and therefore $(K*\mu_j)(x)\to(K*\mu)(x)$ for every $x$.

    It remains to make the convergence locally uniform. For all $x,x'\in\R^d$ and every $j$,
    \[
      |(K*\mu_j)(x)-(K*\mu_j)(x')|\le\int|K(x-y)-K(x'-y)|\,d\mu_j(y)\le L|x-x'|
    \]
    and the same estimate holds for $K*\mu$. Fix a compact set $C\subset\R^d$ and $r>0$, and choose a finite $r$-net $x_1,\ldots,x_n$ of $C$. Then
    \[
      \sup_{x\in C}|(K*\mu_j)(x)-(K*\mu)(x)|
      \le 2Lr+\max_{1\le\ell\le n}|(K*\mu_j)(x_\ell)-(K*\mu)(x_\ell)|
    \]
    The maximum tends to zero as $j\to\infty$ because the net is finite. Letting $r\downarrow0$ proves uniform convergence on $C$.
  \end{proof}

  \section{The Steihaug--Toint truncated-CG subsolver}
  \label{app:steihaug}

  This appendix specifies the unpreconditioned subsolver used in Algorithm~\ref{alg:wtr} \cite{Steihaug1983,ConnGouldToint2000}. Fix one iterate $\rho$ and abbreviate $\mathbb H:=L^2(\rho;\mathbb R^d)$, $\langle u,v\rangle:=\iprho{u}{v}{\rho}$, and $\|u\|:=\normrho{u}{\rho}$. After dropping the constant $\Fcal(\rho)$, the trust-region subproblem is
  \begin{equation}
    \label{eq:st-subproblem}
    \min_{\|s\|\le\Delta}q(s),
    \qquad
    q(s):=\langle g,s\rangle+\frac12\langle s,Bs\rangle
  \end{equation}
  where $g\in\mathbb H$ and $B:\mathbb H\to\mathbb H$ is bounded and self-adjoint but need not be positive semidefinite. We apply $B$ only through products $Bp$.
  For Corollary~\ref{cor:tangent-tr}, use $\mathbb H=T_\rho\Pcal$ with its inherited inner product and $B=\mathsf H_\rho$; the recurrences and proof apply in either Hilbert space.

  For $g\ne0$, the Cauchy point $\bar s$ is the minimizer of $q$ on the feasible part $\{-tg:0\le t\le\Delta/\|g\|\}$ of the negative-gradient ray.
  For any $\|z\|<\Delta$ and nonzero direction $p$, the positive root $t$ of $\|z+tp\|=\Delta$ is
  \begin{equation}
    \label{eq:st-boundary-root}
    \tau_\Delta(z,p)
    :=
    \frac{-\langle z,p\rangle+
      \sqrt{\langle z,p\rangle^2+\|p\|^2(\Delta^2-\|z\|^2)}}{\|p\|^2}
  \end{equation}

  \begin{algorithm}[tbp]
    \caption{Steihaug--Toint truncated CG in a Hilbert space}
    \label{alg:steihaug}
    \begin{algorithmic}[1]
      \Require $g$, products $p\mapsto Bp$ of a bounded self-adjoint operator $B$, radius $\Delta>0$, relative tolerance $\omega\in[0,1)$, CG-iteration limit $J\in\{1,2,\ldots\}$
      \If{$\|g\|=0$}
      \State \Return $0$
      \EndIf
      \State $z_0\gets0$, $r_0\gets g$, $p_0\gets-r_0$
      \For{$j=0,\ldots,J-1$}
      \State $w_j\gets Bp_j$, $\lambda_j\gets\langle p_j,w_j\rangle$
      \If{$\lambda_j\le0$}
      \State \Return $z_j+\tau_\Delta(z_j,p_j)p_j$ \Comment{nonpositive curvature}
      \EndIf
      \State $\alpha_j\gets\|r_j\|^2/\lambda_j$
      \If{$\|z_j+\alpha_jp_j\|\ge\Delta$}
      \State \Return $z_j+\tau_\Delta(z_j,p_j)p_j$ \Comment{boundary truncation}
      \EndIf
      \State $z_{j+1}\gets z_j+\alpha_jp_j$
      \State $r_{j+1}\gets r_j+\alpha_jw_j$
      \If{$\|r_{j+1}\|\le\omega\|g\|$}
      \State \Return $z_{j+1}$ \Comment{interior residual stop}
      \EndIf
      \State $\beta_{j+1}\gets\|r_{j+1}\|^2/\|r_j\|^2$
      \State $p_{j+1}\gets-r_{j+1}+\beta_{j+1}p_j$
      \EndFor
      \State \Return $z_{J}$ \Comment{CG-iteration limit}
    \end{algorithmic}
  \end{algorithm}

  \begin{proposition}[Feasibility and Cauchy decrease]
    \label{prop:steihaug-cauchy}
    In exact arithmetic, let $B$ be bounded and self-adjoint, with $J\ge1$. For $g\ne0$, Algorithm~\ref{alg:steihaug} returns a Steihaug--Toint step $s$ that satisfies $\|s\|\le\Delta$ and
    \begin{equation}
      \label{eq:st-cauchy-decrease}
      q(0)-q(s)
      \ge
      \frac12\|g\|\min\left\{\frac{\|g\|}{\|B\|},\Delta\right\}>0
    \end{equation}
    where $\|g\|/\|B\|:=+\infty$ when $B=0$.
  \end{proposition}

  \begin{proof}
    Every iterate $z_j$ satisfies $\|z_j\|<\Delta$ by the truncation test, and \eqref{eq:st-boundary-root} gives norm exactly $\Delta$ at either truncated return; this proves feasibility. Let $G:=\|g\|$ and $M:=\|B\|$, and set $\widehat t:=\min\{M^{-1},\Delta/G\}$, or $\widehat t:=\Delta/G$ if $M=0$. Since $-\widehat tg$ lies on the feasible ray,
    \begin{equation}
      \label{eq:st-cauchy-ray-bound}
      q(0)-q(\bar s)
      \ge\widehat tG^2-\frac12\widehat t^2\langle g,Bg\rangle
      \ge\widehat tG^2-\frac12M\widehat t^2G^2
      \ge\frac12G\min\{G/M,\Delta\}
    \end{equation}

    Before any exit, induction gives $r_j=g+Bz_j$ and $\langle r_j,p_j\rangle=-\|r_j\|^2$, because $\langle r_{j+1},p_j\rangle=\langle r_j,p_j\rangle+\alpha_j\lambda_j=0$. Hence
    \[
      q(z_j+tp_j)-q(z_j)=-t\|r_j\|^2+\frac12t^2\lambda_j
    \]
    which is nonpositive for $0\le t\le\alpha_j$ when $\lambda_j>0$ and for all $t\ge0$ when $\lambda_j\le0$. Each update uses $t=\alpha_j$, and a boundary truncation uses $t=\tau_\Delta(z_j,p_j)\le\alpha_j$, so the model value never increases. At $j=0$, $p_0=-g$, and $\alpha_0$ is the unconstrained minimizer of $q$ along the ray when $\lambda_0>0$; thus the first update or return is exactly $\bar s$. Therefore $q(s)\le q(\bar s)$, and \eqref{eq:st-cauchy-ray-bound} proves \eqref{eq:st-cauchy-decrease}.
  \end{proof}

\end{appendices}


{\footnotesize
  \begin{thebibliography}{99}

  \bibitem{LouisBolhuisHansen2000}
  A.~A.~Louis, P.~G.~Bolhuis, J.~P.~Hansen,
  \emph{Mean-field fluid behavior of the Gaussian core model},
  Phys.\ Rev.\ E \textbf{62} (2000), 7961--7972,
  doi:10.1103/PhysRevE.62.7961.

  \bibitem{LouisBolhuisHansenMeijer2000}
  A.~A.~Louis, P.~G.~Bolhuis, J.~P.~Hansen, E.~J.~Meijer,
  \emph{Can polymer coils be modeled as ``soft colloids''?},
  Phys.\ Rev.\ Lett.\ \textbf{85} (2000), 2522--2525,
  doi:10.1103/PhysRevLett.85.2522.

  \bibitem{Likos2002} C.~N.~Likos, S.~Rosenfeldt, N.~Dingenouts, M.~Ballauff, P.~Lindner, N.~Werner, F.~V\"ogtle, \emph{Gaussian effective interaction between flexible dendrimers of fourth generation: a theoretical and experimental study}, J.\ Chem.\ Phys.\ \textbf{117} (2002), 1869--1877, doi:10.1063/1.1486209.

  \bibitem{NeumanBlock2004}
  K.~C.~Neuman, S.~M.~Block,
  \emph{Optical trapping},
  Rev.\ Sci.\ Instrum.\ \textbf{75} (2004), 2787--2809,
  doi:10.1063/1.1785844.

  \bibitem{GernertKlapp2015}
  R.~Gernert, S.~H.~L.~Klapp,
  \emph{Enhancement of mobility in an interacting colloidal system
  under feedback control},
  Phys.\ Rev.\ E \textbf{92} (2015), 022132,
  doi:10.1103/PhysRevE.92.022132.

  \bibitem{Arbel2019} M.~Arbel, A.~Korba, A.~Salim, A.~Gretton, \emph{Maximum mean discrepancy gradient flow}, Adv.\ Neural Inf.\ Process.\ Syst.\ \textbf{32} (2019), arXiv:1906.04370.

  \bibitem{Chen2026} Z.~Chen, T.~Karvonen, H.~Kanagawa, F.-X.~Briol, C.~J.~Oates, \emph{Stationary MMD Points}, Proc.\ Mach.\ Learn.\ Res.\ \textbf{306} (2026), 14884--14912.

  \bibitem{TopazBertozzi2004} C.~M.~Topaz, A.~L.~Bertozzi, \emph{Swarming patterns in a two-dimensional kinematic model for biological groups}, SIAM J.\ Appl.\ Math.\ \textbf{65} (2004), 152--174, doi:10.1137/S0036139903437424.

  \bibitem{BernoffTopaz2011} A.~J.~Bernoff, C.~M.~Topaz, \emph{A primer of swarm equilibria}, SIAM J.\ Appl.\ Dyn.\ Syst.\ \textbf{10} (2011), 212--250.

  \bibitem{Kolokolnikov2011} T.~Kolokolnikov, H.~Sun, D.~Uminsky, A.~L.~Bertozzi, \emph{Stability of ring patterns arising from two-dimensional particle interactions}, Phys.\ Rev.\ E \textbf{84} (2011), 015203.

  \bibitem{Balague2013} D.~Balagu\'e, J.~A.~Carrillo, T.~Laurent, G.~Raoul, \emph{Dimensionality of local minimizers of the interaction energy}, Arch.\ Ration.\ Mech.\ Anal.\ \textbf{209} (2013), 1055--1088, doi:10.1007/s00205-013-0644-6.

  \bibitem{Canizo2015} J.~A.~Ca\~nizo, J.~A.~Carrillo, F.~S.~Patacchini, \emph{Existence of compactly supported global minimisers for the interaction energy}, Arch.\ Ration.\ Mech.\ Anal.\ \textbf{217} (2015), 1197--1217, doi:10.1007/s00205-015-0852-3.

  \bibitem{JKO1998} R.~Jordan, D.~Kinderlehrer, F.~Otto, \emph{The variational formulation of the Fokker--Planck equation}, SIAM J.\ Math.\ Anal.\ \textbf{29} (1998), 1--17.

  \bibitem{SalimKorbaLuise2020} A.~Salim, A.~Korba, G.~Luise, \emph{The Wasserstein proximal gradient algorithm}, NeurIPS 2020.

  \bibitem{DiMarino2025} S.~Di~Marino, E.~Naldi, S.~Villa, \emph{Inexact JKO and proximal-gradient algorithms in the Wasserstein space}, 2025, arXiv:2505.23517.

  \bibitem{Bonet2024} C.~Bonet, T.~Uscidda, A.~David, P.-C.~Aubin-Frankowski, A.~Korba, \emph{Mirror and Preconditioned Gradient Descent in Wasserstein Space}, Adv.\ Neural Inf.\ Process.\ Syst.\ \textbf{37} (2024), arXiv:2406.08938, doi:10.52202/079017-0798.

  \bibitem{Detommaso2018} G.~Detommaso, T.~Cui, A.~Spantini, Y.~Marzouk, R.~Scheichl, \emph{A Stein variational Newton method}, NeurIPS 2018, arXiv:1806.03085.

  \bibitem{WangLi2020} Y.~Wang, W.~Li, \emph{Information Newton's flow: second-order optimization method in probability space}, 2020, arXiv:2001.04341.

  \bibitem{Yamamoto2025} N.~Yamamoto, J.~Kim, T.~Suzuki, \emph{Hessian-guided perturbed Wasserstein gradient flows for escaping saddle points}, Adv.\ Neural Inf.\ Process.\ Syst.\ \textbf{38} (2025), arXiv:2509.16974.

  \bibitem{LascuSuzuki2026} R.-A.~Lascu, T.~Suzuki, \emph{From saddle points toward global minima: a Newton-type method on Wasserstein space}, 2026, arXiv:2605.17963.

  \bibitem{Lambert2022} M.~Lambert, S.~Chewi, F.~Bach, S.~Bonnabel, P.~Rigollet, \emph{Variational inference via Wasserstein gradient flows}, Adv.\ Neural Inf.\ Process.\ Syst.\ \textbf{35} (2022), 14434--14447.

  \bibitem{Diao2023} M.~Z.~Diao, K.~Balasubramanian, S.~Chewi, A.~Salim, \emph{Forward-backward Gaussian variational inference via JKO in the Bures--Wasserstein space}, Proc.\ Mach.\ Learn.\ Res.\ \textbf{202} (2023), 7960--7991.

  \bibitem{Terpin2022} A.~Terpin, N.~Lanzetti, B.~Yardim, F.~D\"orfler, G.~Ramponi, \emph{Trust region policy optimization with optimal transport discrepancies: duality and algorithm for continuous actions}, Adv.\ Neural Inf.\ Process.\ Syst.\ \textbf{35} (2022), 19786--19797.

  \bibitem{Toint1988} P.~L.~Toint, \emph{Global convergence of a class of trust-region methods for nonconvex minimization in Hilbert space}, IMA J.\ Numer.\ Anal.\ \textbf{8}(2) (1988), 231--252, doi:10.1093/imanum/8.2.231.

  \bibitem{NocedalWright2006} J.~Nocedal, S.~J.~Wright, \emph{Numerical Optimization}, 2nd ed., Springer, 2006.

  \bibitem{ConnGouldToint2000} A.~R.~Conn, N.~I.~M.~Gould, P.~L.~Toint, \emph{Trust-Region Methods}, SIAM, 2000.

  \bibitem{Absil2007} P.-A.~Absil, C.~G.~Baker, K.~A.~Gallivan, \emph{Trust-region methods on Riemannian manifolds}, Found.\ Comput.\ Math.\ \textbf{7} (2007), 303--330.

  \bibitem{Manns2026} P.~Manns, \emph{Convergence of trust-region algorithms in metric spaces}, SIAM J.\ Optim.\ \textbf{36}(2) (2026), 841--865, doi:10.1137/24M1718627.

  \bibitem{AGS2008} L.~Ambrosio, N.~Gigli, G.~Savar\'e, \emph{Gradient Flows in Metric Spaces and in the Space of Probability Measures}, 2nd ed., Birkh\"auser, 2008.

  \bibitem{Lanzetti2024} N.~Lanzetti, A.~Terpin, F.~D\"orfler, \emph{Variational analysis in the Wasserstein space}, arXiv:2406.10676v2, 2026.

  \bibitem{Otto2001} F.~Otto, \emph{The geometry of dissipative evolution equations: the porous medium equation}, Comm.\ Partial Differential Equations \textbf{26} (2001), 101--174, doi:10.1081/PDE-100002243.

  \bibitem{Gigli2012} N.~Gigli, \emph{Second order analysis on $(\mathcal P_2(M),W_2)$}, Mem.\ Amer.\ Math.\ Soc.\ \textbf{216} (2012), no.~1018.

  \bibitem{Virtanen2020} P.~Virtanen et al., \emph{SciPy 1.0: Fundamental Algorithms for Scientific Computing in Python}, Nat.\ Methods \textbf{17} (2020), 261--272, doi:10.1038/s41592-019-0686-2.

  \bibitem{Steihaug1983} T.~Steihaug, \emph{The conjugate gradient method and trust regions in large scale optimization}, SIAM J.\ Numer.\ Anal.\ \textbf{20}(3) (1983), 626--637, doi:10.1137/0720042.

  \end{thebibliography}
}
\end{document}